\documentclass[twoside]{article}

\usepackage[accepted]{aistats2025}
\runningtitle{VI in location-scale families}
\usepackage{natbib} 
\usepackage{amsfonts}
\usepackage{amsthm}
\usepackage[table, dvipsnames]{xcolor}

\usepackage[colorlinks,linktoc=all]{hyperref} 
\hypersetup{citecolor=MidnightBlue}

\newtheorem{theorem}{Theorem} 
\newtheorem{definition}[theorem]{Definition}
\newtheorem{proposition}[theorem]{Proposition}

\newtheorem{remark}[theorem]{Remark}
\newtheorem{lemma}[theorem]{Lemma}

\newcommand{\KL}{\text{KL}}

\usepackage{graphicx}

\newcommand{\norm}[1]{\left \lVert #1 \right \rVert}

\begin{document}

%

%

\twocolumn[

\aistatstitle{Variational Inference in Location-Scale Families: 
\\ Exact Recovery of the Mean and Correlation Matrix}
\aistatsauthor{ Charles C. Margossian \And Lawrence K. Saul }

\aistatsaddress{ Center for Computational Mathematics \\ Flatiron Institute \And Center for Computational Mathematics \\ Flatiron Institute } ]

\begin{abstract}
  Given an intractable target density $p$, variational inference (VI) attempts to find the best approximation $q$ from a tractable family~$\mathcal Q$.
  This is typically done by minimizing the exclusive Kullback-Leibler divergence, $\KL(q||p)$.
  In practice,~$\mathcal Q$ is not rich enough to contain $p$, and the approximation is misspecified even when it is a unique global minimizer of $\KL(q||p)$.
  In this paper, we analyze the robustness of VI to these misspecifications
  when $p$ exhibits certain symmetries
  and $\mathcal Q$ is a location-scale family that shares these symmetries.
We prove strong guarantees for VI not only under mild regularity conditions but also in the face of severe misspecifications.
Namely, we show that (i) VI recovers the mean of~$p$ 
when~$p$ exhibits an \textit{even} symmetry, and 
(ii) it recovers 
the correlation matrix of $p$ 
when in addition~$p$ 
exhibits an \textit{elliptical} symmetry.
  These guarantees hold for the mean even when $q$ is factorized and $p$ is not, and for the correlation matrix even when~$q$ and~$p$ behave differently in their tails. 
  We analyze various regimes of Bayesian inference where these symmetries are useful idealizations, and we also investigate experimentally how VI behaves in~\mbox{their absence}.  
\end{abstract}

\section{INTRODUCTION}

In many problems, it is necessary to approximate an intractable distribution $p$.
Variational inference (VI) posits a parameterized family $\mathcal Q$ of tractable distributions, and within this family it attempts to find the best approximation to $p$ \citep{Jordan:1999, Wainwright:2008, Blei:2017}. This is typically done by minimizing the exclusive Kullback-Leibler (KL) divergence with respect to the parameters of $\mathcal{Q}$.
An important application of VI lies in Bayesian inference, where the primary computational task is to approximate the posterior distribution.

In practice, $p \notin \mathcal Q$.
For example, we may set $\mathcal Q$ to be a family of factorized distributions---the so-called \textit{mean-field} approximation~\citep{Peterson:1987, Hinton:1993, Parisi:1998}---even when $p$ does not factorize, or to be the family of multivariate Gaussian distributions even when $p$ is non-Gaussian.
Still, while the best approximation in $\mathcal{Q}$ may  not exactly match~$p$, we may hope that it captures certain key features such as its mean and correlation matrix.

Several empirical studies report that VI accurately estimates the mean of $p$; this is true even when the approximation is misspecified to a degree that other quantities, such as the variance and entropy, are poorly estimated
\citep[e.g][]{Mackay:2003, Turner:2011, Blei:2017, Giordano:2018, Margossian:2024}.
These results are nuanced by other examples in which VI poorly estimates expectation values \citep{Huggins:2020, Zhang:2022}. 
There is therefore a need for a theory that explains the empirical successes of VI, but also cautions against its potential failures.

In this paper, we analyze the robustness of VI when~$\mathcal{Q}$ is a location-scale family.
Our first main result is a guarantee for VI's estimate of the mean. In particular,
we show that if the target density~$p$ and each 
approximation $q\!\in\!\mathcal{Q}$ are endowed with a point of symmetry, then a stationary point of $\KL(q||p)$ is found by matching these points of symmetry. Under further conditions, we show that this stationary point is unique and also a global minimizer.
Thus, VI is guaranteed to locate this point, which corresponds to the mean of~$p$, provided we have a well-performing optimizer. These conditions allow for several misspecifications in $\mathcal Q$: for example, $q$ may be factorized while~$p$ is not, or the tails of $q$ may be lighter (or heavier) than those of $p$.
At the same time, the assumption of symmetry is satisfied by many location-scale families 
(e.g., Gaussian, Laplace, student-t). 

Our second main result is a guarantee for VI's estimate of the correlation matrix. We consider the setting
where $p$ exhibits an elliptical symmetry and $\mathcal{Q}$ is a location-scale family whose base density is spherically symmetric. 
Here, under similar assumptions, we show that $\KL(q||p)$ has a unique minimizer that correctly estimates the correlation matrix of $p$.
We note that VI finds this solution even when no $q\!\in\!\mathcal{Q}$ matches the tail behavior of $p$ nor correctly estimates its covariances.

We complement our theoretical results with an empirical study of VI in location-scale families. We experiment with both synthetic targets, in which we control the amount of symmetry in $p$,
and with real-world posteriors from problems in Bayesian inference.
When VI is used for Bayesian inference---to estimate a posterior $p$ over model parameters from data---we observe that this posterior tends to be symmetric in the limits of \textit{both} small and large amounts of data.
Our experiments also probe VI in intermediate regimes where~$p$ is less amenable to symmetric approximations.
Code to reproduce all experimental results and figures in this paper can be found at \url{https://github.com/charlesm93/VI_location_robust}. 

\textbf{Related work.}
There is a rich literature on 
the robustness of VI to misspecifications of 
$\mathcal{Q}$.
%
While we focus on VI's estimates of the mean and correlation matrix, others 
have examined estimates of the marginal likelihood 
\citep{Jordan:1999, Li:2016}, the corresponding maximum likelihood estimator \citep{Wang:2018}, importance weights when~$q$ is used as a proposal distribution \citep{Yao:2018, Vehtari:2024}, and estimators for various measures of uncertainty \citep[e.g.][]{Mackay:2003, Turner:2011, Margossian:2023, Katsevich:2024, Margossian:2024}.
Researchers have also investigated the frequentist properties of variational Bayes estimators:
that is, how well does the first moment of $q$ recover a true value $z^*$ which underlies a data generating process \citep[e.g][]{Alquier:2020, Yang:2020, Zhang:2020}.


Extending this literature, we identify weak conditions under which VI 
\textit{exactly} recovers the mean, including in non-asymptotic regimes, thereby formalizing past empirical observations \citep[e.g][]{Mackay:2003, Blei:2017, Giordano:2018}.
Previous works have used the Wasserstein distance to bound the error of VI's estimates \citep{Huggins:2020,Biswas:2024}.
While these works provide a post-hoc diagnostic, ours provides a complementary theoretical guarantee. 
\citet{Katsevich:2024} also provide theoretical guarantees for the recovery of the mean, in the case where $\mathcal Q$ is the family of Gaussians.
Their results are non-asymptotic, although exact recovery of the mean is only achieved asymptotically.
To our knowledge, our work is the first to provide guarantees for VI with elliptically symmetric distributions and to investigate its ability to recover the correlation matrix.

Many analyses rely on settings where the posterior~$p$ asymptotically approaches a Gaussian, per the Bernstein-von Mises theorem \citep{Vaart:1998}.
Our analysis does not explicitly consider asymptotic limits, nor does it require~$p$ to be Gaussian. Rather, our work leverages the presence in $p$ of even and elliptical symmetries; these symmetries are observed in Gaussian distributions (and in many other distributions besides).
Hence, the techniques in this paper can also be applied to asymptotic analyses of VI that invoke the Bernstein-von Mises theorem.

\section{PRELIMINARIES}

In this section we review the main ideas and assumptions behind VI, as well as how it is used in practice. We then consider particular symmetries that a target density may exhibit and discuss why they are useful idealizations for certain regimes of Bayesian inference.

\subsection{Variational inference in practice}
\label{sec:VI}

The primary computational task in VI is to minimize the exclusive KL divergence
\begin{equation}
    \KL(q||p) = \int \big[\log q(z) - \log p(z)\big] q(z) \text d z,
    \label{eq:KLqp}
\end{equation}
over $q\!\in\!\mathcal Q$, where $\mathcal{Q}$ is a continuously parameterized family of tractable distributions.
Typically, 
it is only possible to evaluate an unnormalized target density~$\tilde{p}$,
and the best approximation in $\mathcal{Q}$ is computed by maximizing the evidence lower bound (ELBO),
\begin{equation}
    \text{ELBO} = \int \big[\log \tilde p(z) - \log q(z)\big] q(z) \text d z.
    \label{eq:ELBO}
\end{equation}
When this integral cannot be evaluated exactly, it can be approximated via Monte Carlo using draws from~$q$.
A similar approach also produces estimates of the gradient of the ELBO with respect to the variational parameters of $\mathcal Q$.
%
Often, it is assumed that the target density $p$ and each $q\!\in\!\mathcal{Q}$ have support over all of $\mathbb R^d$.
This condition for $p$ can be met by transforming constrained latent variables, as in automatic differentiation VI \citep[ADVI;][]{Kucukelbir:2017}; for example, we may apply a logarithmic transformation to a positive~variable.



\subsection{Assumptions for theoretical analysis}

Our goal in this paper is to provide conditions under which VI is guaranteed to correctly estimate certain properties of the target density~$p$. By this, we mean first that the KL divergence in eq.~(\ref{eq:KLqp}) has a unique minimizer in the family $\mathcal{Q}$, and second that this minimizer shares the desired properties of $p$ even when $p\!\notin\!\mathcal{Q}$.
To provide these guarantees, we must make additional assumptions on $\mathcal{Q}$ and $p$. 

\subsubsection{Location-scale families} 

Our analysis begins by supposing that $\mathcal{Q}$ is a location-scale family or a subfamily thereof. 
In the following definition and throughout the paper, we use $S^\frac{1}{2}$ to denote the principal square-root of a positive-definite matrix $S$. 

\begin{definition}
Let $q_0$ be a density over $\mathbb{R}^d$.
A location-scale family $\mathcal{Q}$ is a two-parameter family $\{q_{\nu,S}\}$ of densities over $\mathbb{R}^d$ satisfying
\begin{equation}
q_{\nu,S}(z) = q_0\big(S^{-\frac{1}{2}}(z\!-\!\nu)\big) |S|^{-\frac{1}{2}}
\end{equation}
for all $z,\nu\!\in\!\mathbb{R}^d$ and $d\!\times\!d$ positive-definite matrices $S$. We say that $q_0$ is the base density of the family $\mathcal{Q}$ and that $\nu$ and $S$ are its location and scale parameters.
\end{definition}

\begin{remark}
The square-root of a matrix is not uniquely defined and other choices may be considered.
Our analysis only requires the determinant of the square-root to be positive.
This condition is verified by the principle square-root but also the Cholesky factor of $S$.
\end{remark}

\begin{definition}
A location family $\mathcal{Q}$ is a one-parameter subfamily $\{q_\nu\}$ of location-scale densities that share the same scale parameter.
\end{definition}

Location-scale families are used widely in statistics: examples include
the Gaussian, Laplace, student-t, and Cauchy families.
They are also particularly relevant to black box VI~\citep[e.g][]{Ranganath:2014, Kucukelbir:2017, Cai:2024}, where the variational approximation is often taken to be Gaussian.
Location-scale distributions are simple to manipulate, and they can be sampled using the
\textit{reparameterization trick} \citep{Kingma:2014, Rezende:2014, Titsias:2014}, which produces low variance Monte Carlo estimators of the ELBO and its gradient.
%

\subsubsection{Even and elliptical symmetries} 

Our analysis focuses on the setting where the target density $p$ and the base density $q_0\!\in\!\mathcal{Q}$ exhibit certain symmetries. We define these symmetries next.

\begin{definition} (Even and odd symmetry.)
  We say a function $f: \mathbb R^d \to \mathbb R$ is even-symmetric about a location point $\nu \in \mathbb R^d$ if for all $\zeta\!\in\!\mathbb{R}^d$ it satisfies
  \begin{equation}
    f(\nu + \zeta) = f(\nu\!-\!\zeta). 
  \end{equation}
  Similarly, we say a function $f$ is odd-symmetric about a location point $\nu$ if for all $\zeta\!\in\!\mathbb{R}^d$ it satisfies
  \begin{equation}
    f(\nu + \zeta) = - f(\nu\!-\!\zeta). 
  \end{equation}
\end{definition}

\begin{definition}
  (Spherical symmetry.)
  We say a density $p(\zeta)$ is spherically symmetric if the density can be written as a function $\norm \zeta$; that is, if $\norm {\zeta_1} = \norm{\zeta_2}$, then $p(\zeta_1) = p(\zeta_2)$.
\end{definition}

\begin{definition}
  (Elliptical symmetry.)
  We say a density $p(z)$ is elliptically symmetric about $\nu\!\in\!\mathbb R^d$ if there exists a positive-definite matrix $M\! \in\! \mathbb R^{d \times d}$ such that the density of $\zeta = M^{-\frac{1}{2}} (z\!-\!\nu)$ is spherically symmetric. In this case we call $M$ the scale matrix of $p$.
\end{definition}

\begin{remark}
\label{rem:corr}
    If $p(z)$ is elliptically symmetric with scale matrix $M$, then $\text{Corr}_p[z_i,z_j] = M_{ij}/\sqrt{M_{ii}M_{jj}}$.
\end{remark}

We emphasize, per the previous remark, that the correlation matrix of an elliptically symmetric distribution is determined by its scale matrix, even though its covariance matrix is not.

Assumptions of symmetry form the cornerstone of our analysis. Specifically, we analyze VI in cases where 
(i) $p$ is even-symmetric and $\mathcal{Q}$ is a location family whose base density $q_0$ is also even-symmetric, and (ii) $p$ is elliptically symmetric and $\mathcal{Q}$ is a location-scale family whose base density $q_0$ is spherically symmetric. 

\subsubsection{Regularity conditions on $p$}
\label{sec:regularity}

VI recasts the problem of inference as one of optimization. It is generally difficult to prove guarantees for non-convex problems in optimization, and this difficulty is also present in our setting. To prove the strongest guarantees for VI, our analysis places further assumptions on the target density: specifically, we assume that $\log p$ is concave on all of $\mathbb{R}^d$ and strictly concave on some open set of $\mathbb{R}^d$. We also assume that~$p$ is differentiable, and that $\log p$ and each $q\!\in\!\mathcal{Q}$ satisfy the assumptions of the dominated convergence theorem to permit differentiating under the integral sign. (More concretely, the reader may assume that each $q\!\in\!\mathcal{Q}$ has finite moments of all orders, and that $\|\nabla \log p(z)\|$ can be bounded by some polynomial in $\|z\|$.)

\subsection{Symmetries and Bayesian inference} \label{sec:posterior-symmetry}

We now briefly examine how exact or approximate symmetries of the target density, $p$, arise in various regimes of Bayesian inference.
%
%
In a Bayesian analysis, we seek to compute the posterior distribution $\pi(z|x)$ of a latent random variable $z$ conditioned on some observation (or collection of observations) $x$. This posterior distribution is given by
%
\begin{equation}  \label{eq:Bayes}
  \pi(z|x) \propto \pi(z) \pi(x|z),
\end{equation}
where $\pi(z)$ is the prior distribution, and $\pi(x|z)$, the likelihood function. When VI is used for Bayesian analysis, the target density we approximate is \mbox{$p(z) = \pi(z|x)$}, and 
eq.~\eqref{eq:Bayes} reveals how $p$ may come to exhibit exact or approximate symmetries.

%

  \textbf{Common symmetry.} The posterior $\pi(z|x)$ will inherit any common symmetry of the prior $\pi(z)$ and the likelihood $\pi(x|z)$ --- for example, if the latter are both even or elliptically symmetric \textit{about the same point}. This scenario is unlikely, although it can arise in empirical Bayes \citep{Carlin:2000}.
  If $\pi(x|z)$ and $\pi(z)$ are symmetric about different points, then the posterior $\pi(z|x)$ will not in general exhibit any symmetry. An exception occurs when both $\pi(x|z)$ and $\pi(z)$ are Gaussian, in which case $\pi(z|x)$ is also Gaussian.

  \textbf{Dominating prior.} If the prior $\pi(z)$ is symmetric and ``dominates'' the likelihood, i.e. $\pi(z|x) \approx \pi(z)$, then the posterior inherits its symmetry to some degree. This occurs if the data is sparse and the likelihood uninformative.

  \textbf{Dominating likelihood.} Conversely, suppose the likelihood dominates.
  This can occur if the prior $\pi(z)$ is flat (i.e. independent of $z$) or we have a large collection of observations $x$ 
  such that $\pi(z|x)$ is very strongly peaked.
  In particular, we know per the Berstein-von-Mises theorem that asymptotically, and under certain regularity conditions, the posterior becomes provably Gaussian (and hence even-symmetric) in the limit of a large number of independent observations. We note, however, that the likelihood $\pi(x|z)$ for a \textit{finite} number of observations does not generally exhibit a point of symmetry even when those observations are individually generated by even-symmetric distributions (e.g., Cauchy, Laplace); this breakdown occurs because the product of non-Gaussian distributions with \textit{different} points of even-symmetry does not itself, in general, exhibit a point of even symmetry. 
   Empirically, though, we find that such likelihoods tend to be approximately symmetric provided that the data is non-adversarial; see Appendix~\ref{app:symmetry} for more discussion.
   
\begin{figure}
    \centering
    \includegraphics[width=1\linewidth]{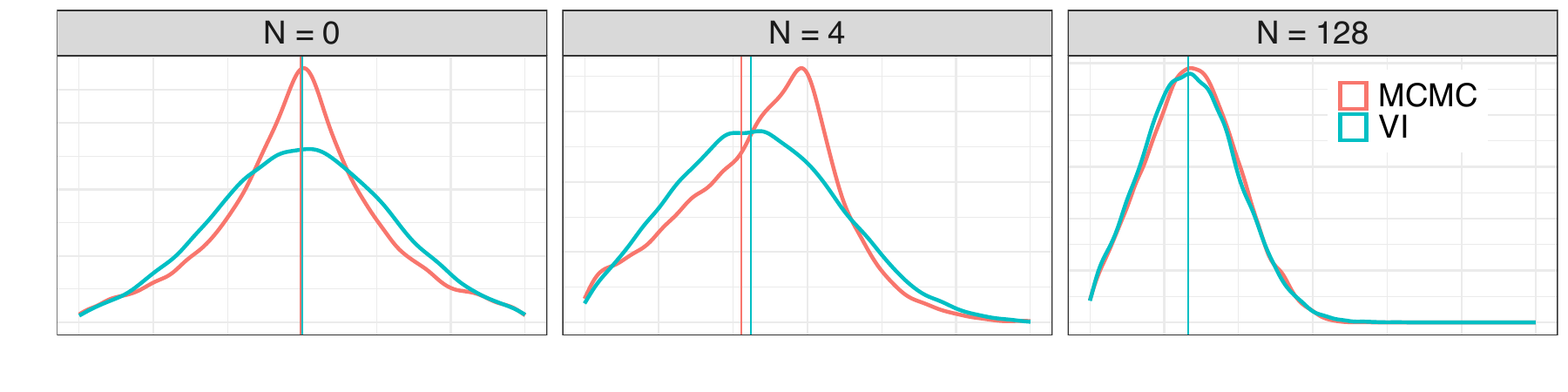}
    \caption{\textit{Posterior distribution of $\beta_1$ for a Bayesian logistic regression
    with~$N$ examples.
    Vertical lines indicate the means estimated by MCMC and VI.
    These estimates match when the posterior is symmetric \mbox{$(N\!=\!0,128)$} and differ when it is not $(N\!=\!4)$.
    }}
    \label{fig:logistic_symmetry}
\end{figure}

Fig.~\ref{fig:logistic_symmetry} illustrates how the symmetry of a posterior depends on the regime (sparse versus rich data) of Bayesian inference, and specifically on the number of observations, $N$, in a Bayesian logistic regression. We consider the model
\begin{eqnarray}
    \beta_0, \beta_1, \beta_2 & \overset{\text{iid}}{\sim} & \text{Laplace}(0, 0.5), \nonumber \\
    y_i & \sim & \text{Bernoulli}(\text{logit}^{-1}(\boldsymbol \beta^T {\bf x}_i)),
\end{eqnarray}
where $\boldsymbol \beta = (\beta_0, \beta_1, \beta_2)$ and ${\bf x}_{i} = (1, x_{i1}, x_{i2})$ is the $i^\text{th}$ component of a design matrix with two covariates.
We compare the posteriors estimated by ADVI and, as a benchmark, long runs of Markov chain Monte Carlo (MCMC) for different numbers of examples.
Both algorithms are implemented in the statistical software \texttt{Stan} \citep{Carpenter:2017}; 
see Appendix~\ref{app:algorithms} for details of these algorithms and Appendix~\ref{app:logistic} for additional experimental results.
Figure~\ref{fig:logistic_symmetry} plots these estimated posterior distributions for $\beta_1$ and $N\!=\!0,4,128$ examples.
We see that the true posterior, as reported by MCMC, is even-symmetric for $N\! =\! 0$ (where the posterior reduces to the prior) and very nearly symmetric for $N\! =\! 128$, but clearly asymmetric for $N\! =\! 4$, a regime where neither the prior nor the likelihood dominates.

Fig.~\ref{fig:logistic_symmetry} illustrates another result: the posterior mean is accurately estimated by VI when the posterior is symmetric ($N\!=\!0$ and $N\!=\!128$), but less accurately estimated when it is not ($N\! =\! 4$). We now formalize this result, providing guarantees for VI when the target density is even-symmetric.



\section{EXACT RECOVERY OF MEAN}

In this section we prove that there are fairly broad conditions under which VI exactly recovers the mean of the target density. We also use simple examples to illustrate the scope of this result and various caveats. In all that follows, we assume that $p$ and the base density $q_0\!\in\!\mathcal{Q}$ have support on all of $\mathbb{R}^d$; we also assume the regularity conditions on $p$ in section~\ref{sec:regularity}.


\subsection{Theoretical Statement} \label{sec:theory-location}

Our first theorem shows how VI is able to capitalize on an even symmetry of the target density.

\begin{theorem}[Exact Recovery of Mean]
\label{thm:location}
Let $\mathcal{Q}$ be a location family whose base distribution $q_0$ is even-symmetric about the origin. If $p$ is even-symmetric about $\mu$, then $\KL(q_\nu||p)$ has a stationary point at \mbox{$\nu\!=\!\mu$}; furthermore, if $\log p$ is concave on $\mathbb{R}^d$ and strictly concave on some open set of $\mathbb{R}^d$, then this stationary point is a unique minimizer of~$\KL(q_\nu||p)$.
\end{theorem}

\begin{proof}
We start by exploiting the parameterized form of the location family $\mathcal{Q}$. For $q_\nu\!\in\!\mathcal{Q}$, we can compute
\begin{align}
\KL(q_\nu||p) 
  &= \int\!\! \big[\log q_\nu(z) - \log p(z)\big]\, q_\nu(z)\,\text{d}z, \label{eq:KL-nu}\\
 &= \int\!\! \big[\log q_0(\zeta) - \log p(\nu\!+\!\zeta)\big]\, q_0(\zeta)\,\text{d}\zeta, \label{eq:shift} \\
 &= -\mathcal{H}(q_0)\, - \int\! \log p(\nu\!+\!\zeta)\, q_0(\zeta)\,\text{d}\zeta \label{eq:KL-convex-nu}
\end{align}
where in eq.~(\ref{eq:shift}) we have shifted the variable of integration (to $\zeta\!=\!z\!-\!\nu$), and in eq.~(\ref{eq:KL-convex-nu}), we have used $\mathcal{H}(q_0)$ to denote the entropy of~$q_0$. Note that $\mathcal{H}(q_0)$ does not depend on the location parameter $\nu$. The gradient of eq.~(\ref{eq:KL-convex-nu}) with respect to~$\nu$ is therefore given~by
\begin{align}
\nabla_\nu \KL(q_\nu||p)
  &= -\int\!\! \nabla_\nu\big[\log p(\nu\!+\!\zeta)\big]\,q_0(\zeta)\,\text{d}\zeta, \label{eq:diff-nu-thru} \\
  &= -\int\!\! \nabla_\zeta\big[\log p(\nu\!+\!\zeta)\big]\,q_0(\zeta)\,\text{d}\zeta \label{eq:grad-swap},
\end{align}
where in eq.~(\ref{eq:diff-nu-thru}) we have differentiated under the integral sign, and in eq.~(\ref{eq:grad-swap}) we have exploited the symmetric appearance of $\zeta$ and $\nu$ inside the gradient. Now suppose $\nu\!=\!\mu$. In this case the gradient is given~by
\begin{equation}
  \nabla_\nu \KL(q_\nu||p)\Big|_{\nu=\mu}\!\! = -\int\!\! \nabla_\zeta\big[\log p(\mu\!+\!\zeta)\big]\,q_0(\zeta)\, \text{d}\zeta.
  \label{eq:grad-nu-integral}
\end{equation}
By assumption $p$ is even-symmetric about $\mu$, satisfying $p(\mu\!+\!\zeta) = p(\!\mu-\!\zeta)$. It follows that the gradient in the integrand of eq.~(\ref{eq:grad-nu-integral}) is odd-symmetric, with
\begin{equation}
   \nabla_\zeta\big[\log p(\mu\!+\!\zeta)\big] = -\nabla_\zeta\big[\log p(\mu\!-\!\zeta)\big].
\end{equation}
Also recall, by assumption, that $q_0$ is even-symmetric about the origin, satisfying $q_0(\zeta) = q_0(-\zeta)$.
Thus the integrand in eq.~(\ref{eq:grad-nu-integral}) is the product of an odd-symmetric and even-symmetric function, and hence an odd-symmetric function in its own right; accordingly, the integral must vanish when it is performed over all of $\mathbb{R}^d$. This proves the first claim of the theorem that $\nabla_\nu \KL(q_\nu||p)\!=\!0$ when $\nu\!=\!\mu$.

Now suppose that $\log p$ is concave on $\mathbb{R}^d$ and strictly concave on some open set of $\mathbb{R}^d$. Under these assumptions, we show in Lemma~\ref{lemma:KL} of appendix~\ref{app:lemma-KL} that $\KL(q_\nu||p)$ is  strictly convex in~$\nu$, and hence the stationary point at $\nu\!=\!\mu$ corresponds to a unique minimizer. This proves the theorem.
\end{proof}

\begin{remark}
\label{remark:more-parameters}
VI also recovers the mean of $p$ when $\mathcal{Q}$ is a location-scale family or a subfamily thereof (e.g., one where the scale matrix is restricted to be diagonal). In this setting, the same steps yield a unique minimizer at $\nu\!=\!\mu$ for any value of the scale parameter.
\end{remark}

\subsection{Illustrative examples}  \label{sec:location-illustration}

\begin{figure}
    \centering
    \includegraphics[width=1.6in]{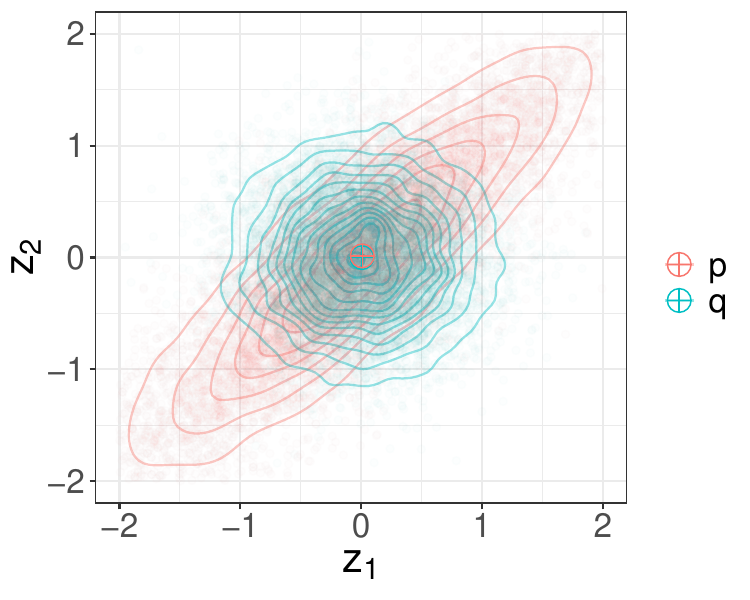}
    \includegraphics[width=1.6in]{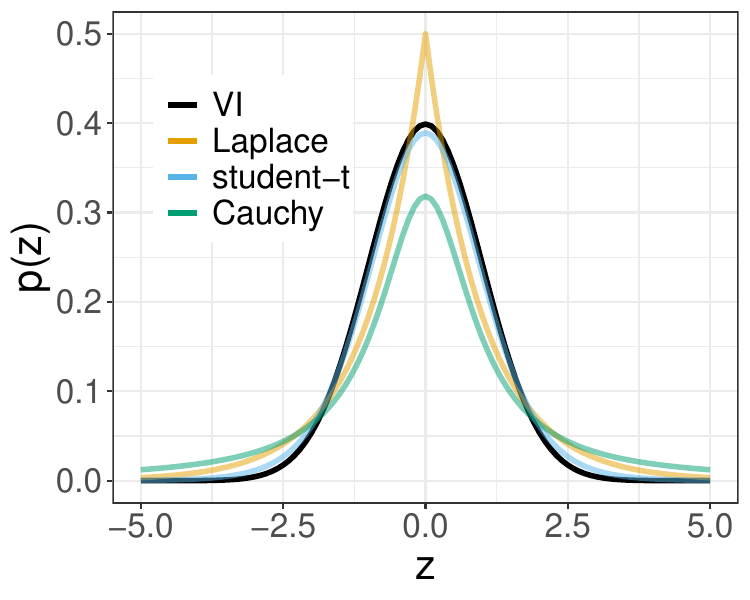}
    \caption{\textit{
    Robustness of VI to misspecifications. Left: VI with a factorized Gaussian, $q$, exactly recovers the mean of a multivariate student-t, $p$. Right: VI with a univariate Gaussian recovers the point of symmetry in target densities with different tails. This point equals the mean when the target is Laplace or student-t, but not when it is Cauchy (whose mean does not exist).}}

    \label{fig:tail_robust}
\end{figure}

The content of the theorem is best illustrated by example. We emphasize that the conditions of the theorem 
allow for severe misspecifications in which the target density $p$ cannot be well approximated over its whole domain by any $q\!\in\!\mathcal{Q}$. Our first examples illustrate how VI behaves when $\mathcal{Q}$ is misspecified but still satisfies the conditions of the theorem---in particular, when $q$ is factorized, but~$p$ is not, and when $p$ is heavy-tailed, but $q$ is not.

\paragraph{Non-factorized target with heavy tails.}  Figure~\ref{fig:tail_robust} (\textit{left}) illustrates how VI behaves when $\mathcal{Q}$ is the family of Gaussians with diagonal covariance matrices and $p$ is a heavy-tailed two-dimensional \mbox{student-t} with $k\! =\! 10$ degrees of freedom and correlation 0.9 between $z_1$ and~$z_2$.
Despite these misspecifications, $p$ and $\mathcal Q$ satisfy all the conditions of the theorem, and VI exactly recovers the mean of $p$.


\paragraph{Other heavy-tailed targets.} Figure~\ref{fig:tail_robust} (\textit{right}) illustrates how VI behaves when $\mathcal{Q}$ is the family of one-dimensional Gaussians and $p$ is (i) a Laplace distribution, (ii) a student-t distribution with $k\!=\!10$ degrees of freedom, and (iii) a Cauchy distribution.
In all of these cases, VI recovers the point of symmetry of~$p$; this point is equal to the mean of $p$ for the Laplace and student-t distributions, and to the median of $p$ for the Cauchy distribution (whose mean does not exist).





\begin{figure}
  \centering
  \includegraphics[width=0.49\linewidth]{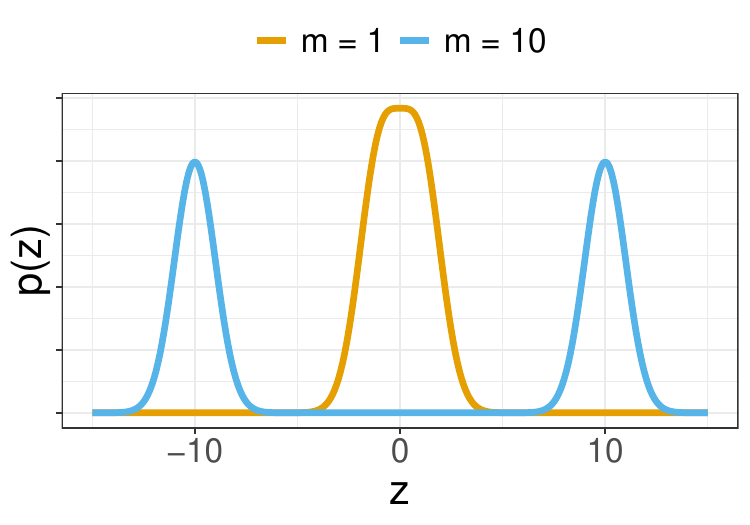}
  \includegraphics[width=0.48\linewidth]{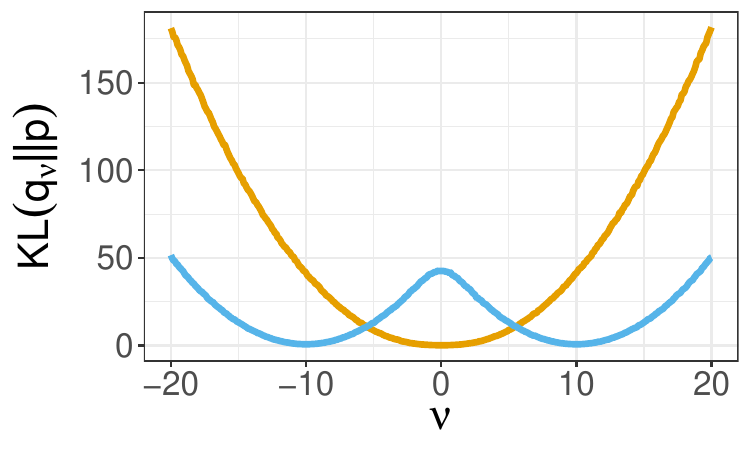}
  \caption{\textit{Approximating the mixture $p$ of two normals in eq.~(\ref{eq:mixture}) when the modes are close ($m\!=\!1$) or well separated ($m\!=\!10$).
  Left: probability densities.
  Right: KL divergence between $p$ and a single normal $q_\nu$ with mean $\nu$.
  Per Theorem~\ref{thm:location}, $\KL(q_\nu||p)$ has a stationary point at $\nu\!=\!0$. This stationary point is a minimizer when $m\!=\!1$ but a local maximizer when~\mbox{$m\!=\!10$}.}}
  \label{fig:mixture}
\end{figure}

Our next examples illustrate how VI behaves when the conditions of the theorem are \textit{not} satisfied.

\paragraph{Bimodal target.} Figure~\ref{fig:mixture} illustrates how VI behaves when fitting a normal approximation to a balanced mixture of two univariate Gaussians:
\begin{equation}
  p(z) = 0.5 \ \mathcal N(z; -m, 1) + 0.5 \ \mathcal N(z; m, 1),
  \label{eq:mixture}
\end{equation}
In this example, $p$ is symmetric about the origin, but it is not log-concave, thus violating one of the theorem's conditions. The figure illustrates two cases---one where the modes of $p$ are close, with $m\!=\!1$ in eq.~(\ref{eq:mixture}), and another, where they are well separated, with $m\!=\!10$.
Per the theorem, $\KL(q_\nu||p)$ has a stationary point at the origin. However, when $p$ has well-separated modes, this stationary point is a not a local minimizer of $\KL(q||p)$; instead it is a local \textit{maximizer}. In this case, VI would not center its approximation at the origin, but instead at one of the two modes.

\begin{figure}
  \begin{center}
  \includegraphics[width=3.25in]{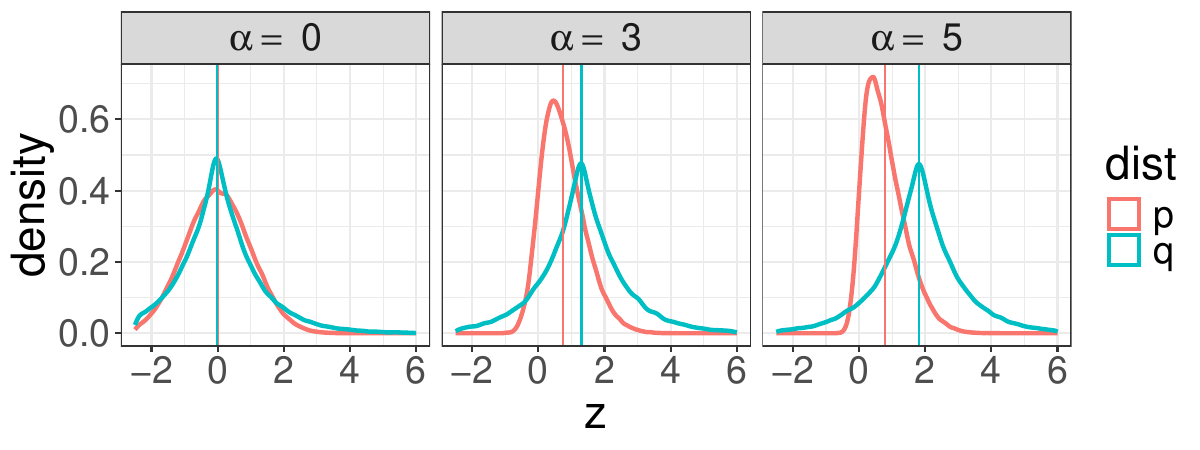}
  \caption{\textit{VI approximation of a skewed normal distribution, $p$, by a Laplace distribution, $q$. 
  The vertical lines indicate the mean of each distribution.
  Per Theorem~\ref{thm:location}, VI correctly estimates the mean of $p$ when $\alpha\!=\!0$ and $p$ is symmetric. 
  However, VI's estimate worsens as $\alpha$ increases and $p$ becomes less symmetric.}}
  \label{fig:skewed}
  \end{center}
\end{figure}

\paragraph{Asymmetric target.} 
Figure~\ref{fig:skewed} illustrates how VI misestimates the mean when $p$ is not symmetric. In this example, we consider
\begin{eqnarray} 
  p(z) &=& \text{skewed-Normal}(z ; 0, 1, \alpha) \nonumber \\
  q(z) & = & \text{Laplace}(z ; \nu, 1),
\end{eqnarray}
where $\alpha \in \mathbb R$ controls the skewness of $p$,
and the location $\nu$ of $q$ is obtained by minimizing $\KL(q||p)$. 
Note that in this example, because $q$ is Laplace, the variational approximation is misspecified even when $p$ is not skewed.
The greater the skewness of $p$, the less accurate is VI's estimate of the mean. 

\section{EXACT RECOVERY OF CORRELATION MATRIX}

Next we prove that under fairly broad conditions VI can exactly recover the correlation matrix of the target density. Here again we assume that $p$ and the base density $q_0\!\in\!\mathcal{Q}$ have support on all of $\mathbb{R}^d$; we also assume the regularity conditions on $p$ in section~\ref{sec:regularity}.


\subsection{Theoretical statement}

Our second theorem shows how VI is able to capitalize on an elliptical symmetry of the target density.

\begin{theorem}[Exact Recovery of Correlation Matrix]
\label{thm:scale}
Let $\mathcal{Q}$ be a location-scale family whose base distribution $q_0$ is spherically symmetric. If $p$ is elliptically symmetric about $\mu$ with scale matrix $M$, and if $\log p$ is concave on $\mathbb{R}^d$ and strictly concave on some open set of $\mathbb{R}^d$, then $\KL(q||p)$ has a unique minimizer with respect to the location-scale parameters $(\nu,S)$ of~$\mathcal{Q}$ at $\nu\!=\!\mu$ and $S=\gamma^2 M$ for some $\gamma\in\mathbb{R}$.
Per Remark~\ref{rem:corr}, the correlation matrix is recovered at this minimizer. 
\end{theorem}

\begin{proof}
Let $q\!\in\!\mathcal{Q}$ with location-scale parameters $(\nu,S)$. Write $q(z)\!=\!q_0(\zeta)|S|^{-\frac{1}{2}}$ where $\zeta\! =\! S^{-\frac{1}{2}}(z\!-\!\nu)$.
Then, analogous to eqs.~(\ref{eq:KL-nu}--\ref{eq:KL-convex-nu}), we can compute
\begin{align}
\KL(q||p) 
 &= -\mathcal{H}(q_0)\, - \tfrac{1}{2}\log|S|
 \nonumber \\ 
 & \mbox{\hspace{4ex}} - \int\! \log p\Big(S^{\frac{1}{2}}\zeta\!+\!\nu\Big)\,q_0(\zeta)\,\text{d}\zeta.
 \label{eq:KL1-scale}
 \end{align}
By the assumption of elliptical symmetry, there exists a spherically symmetric density $p_0$ such that 
\begin{equation}
p(z) = p_0\Big(M^{-\frac{1}{2}}(z\!-\!\mu)\Big)|M|^{-\frac{1}{2}}.
\label{eq:p0-spherical}
\end{equation}
Substituting the above into eq.~(\ref{eq:KL1-scale}), we obtain a generalization of eq.~(\ref{eq:KL-convex-nu}) to location-scale  families:
\begin{align*}
\KL(q||p) 
  &= -\mathcal{H}(q_0) -\tfrac{1}{2}\log|S| + \tfrac{1}{2}\log|M| \nonumber \\
 & \mbox{\hspace{2ex}} - \int\! \log p_0\Big(M^{-\frac{1}{2}}\big[S^{\frac{1}{2}}\zeta\!+\!\nu\!-\!\mu\big]\Big)\,q_0(\zeta)\,\text{d}\zeta.
 \end{align*}
Since~$p_0$ and $q_0$ are spherically symmetric, they are also even-symmetric about the origin; per Theorem~\ref{thm:location}, the above is minimized with respect to the location parameter when $\nu\!=\!\mu$. Eliminating~$\nu$, we find
\begin{align}
\KL(q||p) 
  &= -\mathcal{H}(q_0) -\tfrac{1}{2}\log|S| + \tfrac{1}{2}\log|M| \nonumber \\
  & \vspace{3ex} -\, \int\! \log p_0\big(M^{-\frac{1}{2}}S^{\frac{1}{2}}\zeta\big)\,q_0(\zeta)\,\text{d}\zeta.
  \label{eq:KL2-scale}
 \end{align}
 To prove the theorem we must minimize the right side of eq.~(\ref{eq:KL2-scale}) with respect to the scale parameter $S$, or equivalently, with respect to the matrix~$J=M^{-\frac{1}{2}} S^\frac{1}{2}$. In terms of this matrix, eq.~(\ref{eq:KL2-scale}) simplifies to
\begin{equation}
\KL(q||p) 
  = -\mathcal{H}(q_0) -\log|J| - \int\! \log p_0(J\zeta)\,q_0(\zeta)\,\text{d}\zeta.
  \label{eq:KL3-scale}
\end{equation}
Eq.~(\ref{eq:KL3-scale}) shows that after eliminating the location parameter, the remaining objective for VI is strictly convex in~$J$; this property follows the strict concavity of $\log|J|$ and the assumption that $p$ (and hence also $p_0$) is log-concave. Hence, any stationary point of eq.~(\ref{eq:KL3-scale}) corresponds to a unique global minimizer. We will show that such a stationary point occurs when
 \begin{equation}
 J=\gamma I
 \end{equation}
for some $\gamma\!>\!0$. The rest of the proof uses the symmetry of $p_0$ and $q_0$ to verify this ansatz, which in turn implies that $S\!=\!\gamma^2 M$.

Since $p_0$ and $q_0$ are spherically symmetric, we can define functions $f,g:\mathbb{R}\rightarrow\mathbb{R}$ by the relations
\begin{align}
f(\|J\zeta\|) & = \log p_0(J\zeta), \label{eq:fzeta}\\
g(\|\zeta\|) &= q_0(\zeta).
\end{align}
We can also compute the gradient of eq.~(\ref{eq:KL3-scale}) with respect to $J$ in terms of these functions; it is given by
\begin{equation}
\nabla_J \KL(q||p) = -J^{-1} - \int\! f'(\|J\zeta\|)\,\frac{J\zeta\zeta^\top}{\|J\zeta\|}\, g(\|\zeta\|)\, \text{d}\zeta.
\end{equation}
If a minimizer exists at $J\!=\!\gamma I$, then this gradient must vanish there; equivalently, it must be the case that
\begin{equation}
    \gamma^{-1} I = -\int\!f'(\gamma\|\zeta\|)\frac{\zeta\zeta^\top}{\|\zeta\|}\,g(\|\zeta\|)\,\text{d}\zeta.
\label{eq:lambda}
\end{equation}
We complete the proof by showing that eq.~(\ref{eq:lambda}) determines a unique solution for some $\gamma\!>\!0$.


First we verify that the right side of eq.~(\ref{eq:lambda}) is a scalar multiple of the identity matrix.
To do so, we consider how to perform the $i^\text{th}$ component of the integration,
\begin{equation}
  \int \text{d}\zeta_{\backslash\{i\}}  
  \int \text{d}\zeta_i\, \Big[f'(\gamma\|\zeta\|)\,g(\|\zeta\|)\,\|\zeta\|^{-1}\Big] \zeta_j \zeta_i
\end{equation}
and note that the bracketed term in the integrand is spherically symmetric (and hence also invariant under the change of variables $\zeta_i\rightarrow-\zeta_i$). Now if $i\!\neq\!j$, then the integrand as a whole is an odd-symmetric function of $\zeta_i$, and hence the integral over all positive and negative values of $\zeta_i$ vanishes. On the other hand, if $i\!=\!j$, then from the spherical symmetry of the bracketed term it follows that the integral has the same value for all $i$. This implies, in turn, that all the diagonal elements on the right side of eq.~(\ref{eq:lambda}) are equal.

%
%
%
%
%

It remains to prove that a solution exists for $\gamma$. Since each side of eq.~(\ref{eq:lambda}) is a scalar multiple of the identity matrix, we can solve for $\gamma$ by equating their traces:
\begin{eqnarray}
d\gamma^{-1} & = & -\int\!f'(\gamma\|\zeta\|)\,g(\|\zeta\|)\,\|\zeta\|^{-1}\,\sum_{i} \zeta^2_i\,\text{d}\zeta \nonumber \\
& = & - \int\!f'(\gamma\|\zeta\|)\,g(\|\zeta\|)\,\|\zeta\|\,\text{d}\zeta.
\label{eq:traces}
\end{eqnarray}
 We proceed by evaluating the integral 
in spherical coordinates with $r\!=\!\|\zeta\|$.
After integrating out the angular coordinates, 
eq.~(\ref{eq:traces}) reduces to
\begin{equation}
d\gamma^{-1} = -\frac{2\pi^{\frac{d}{2}}}{\Gamma \left(\frac{d}{2} \right)}\!\int\!f'(\gamma r)\,g(r)\,r^d \text{d}r,
\label{eq:spherical-integral}
\end{equation}
where we leveraged the fact the surface area of a $d$-dimensional hypersphere of radius $r$ is 
\mbox{$2\pi^{\frac{d}{2}} r^{d-1}/\Gamma \left(d/2 \right)$}. 
Note that the left side of eq.~(\ref{eq:spherical-integral}) decreases monotonically with $\gamma$ from infinity to zero; thus the right side will determine a unique solution if it is a positive \textit{non-decreasing} function of~$\gamma$.
Now, the log concavity of $p$ implies that the derivative of $\log p$ is monotone.
Then, it follows from eqs.~(\ref{eq:p0-spherical}) and~(\ref{eq:fzeta}) that $-f'(\gamma r)$ is a non-decreasing function of~$\gamma$. It must also be true that $-f'(0)\geq 0$, since otherwise $\log p$ would have a non-concave kink at the origin. Thus the right side of eq.~(\ref{eq:spherical-integral}) is indeed a positive non-decreasing function of~$\gamma$, and a unique solution for $\gamma$ is determined.
 \end{proof}

\subsection{Illustrative examples}

\begin{figure}
    \centering
    \includegraphics[width=0.49\linewidth]{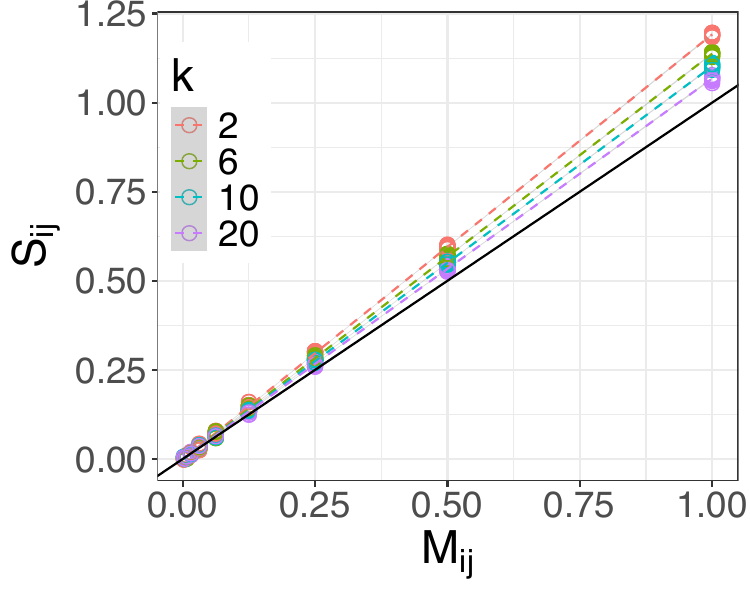}
    \includegraphics[width=0.49\linewidth]{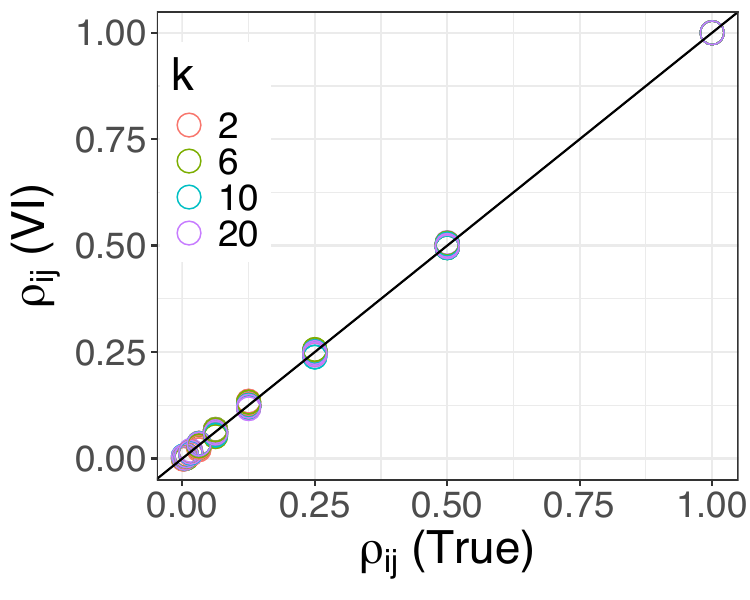}
    \caption{\textit{Gaussian approximation of VI to a multivariate student-t with varying degrees of freedom, $k$.
    Left: VI's scale matrix $S$ equals the target's scale matrix $M$, up to a multiplicative constant, which varies with $k$.
    Right: for all $k$, VI exactly recovers the elements $\rho_{ij}$ of the correlation matrix.
    }   
    }
    \label{fig:multi-student-t}
\end{figure}


Figure~\ref{fig:multi-student-t} illustrates how VI with a Gaussian scale-location family recovers the correlation matrix of multivariate Student-t target densities in $\mathbb{R}^{10}$. We set
\begin{align}
    p(z) &= \text{multi-student-t}(z;k,0,M), \\
    q(z) &= \mathcal N(z ; \nu, S), \label{eq:gvi}
\end{align}
where $k$ denotes the number of degrees of freedom in the Student-t target, $M$ its scale matrix, and $S$ the covariance matrix of the Gaussian approximation.
In this example, the elements of $M$ were one along the diagonal and constant off the diagonal ($S_{ij} = 0.9$ for $i\neq j$).
The results in the figure were obtained by ADVI for target densities with heavier tails ($2\!\leq\!k\!\leq\!20$) than $q$, particularly for small values of $k$. The results verify Theorem~\ref{thm:scale}: $S$ matches~$M$ up to a multiplicative constant and thus yields the exact correlation matrix of~$p$.

Figure~\ref{fig:scale_recovery} (right) shows how VI behaves on the two-dimensional Rosenbrock or \textit{crescent} distribution,
%
%
\begin{align} \label{eq:crescent}
   p(z_1) &= \mathcal N(0, 10^2), \nonumber \\
   p(z_2|z_1) &= \mathcal N(0.03 (z_1^2\!-\! 100), 1),
\end{align}
which does \textit{not} have elliptical symmetry. In this case, the approximation in eq.~(\ref{eq:gvi}) 
recovers neither the mean nor the correlation; see also Table~\ref{tab:targets}. Note, however, that the Rosenbrock distribution is symmetric about the line $z_1\!=\!0$, and that VI does correctly recover the first component of its mean.

\begin{figure}
    \centering
    \includegraphics[width=1.6in]{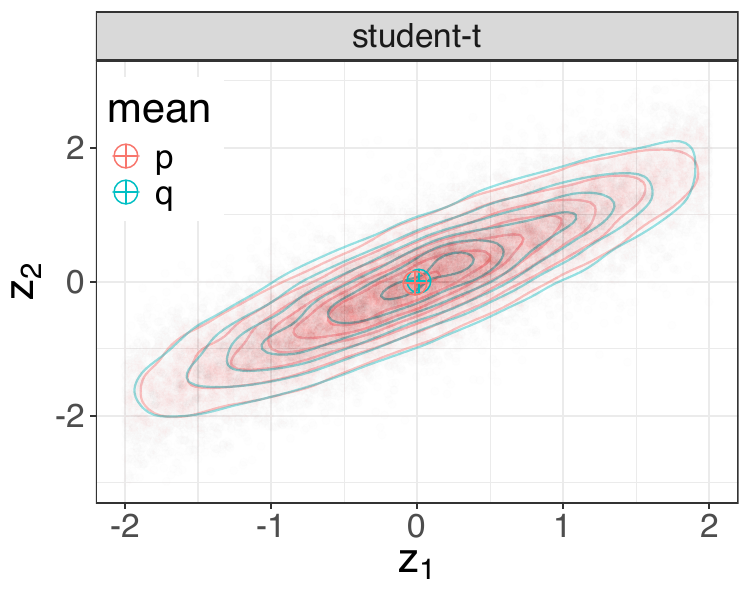}
    \includegraphics[width=1.6in]{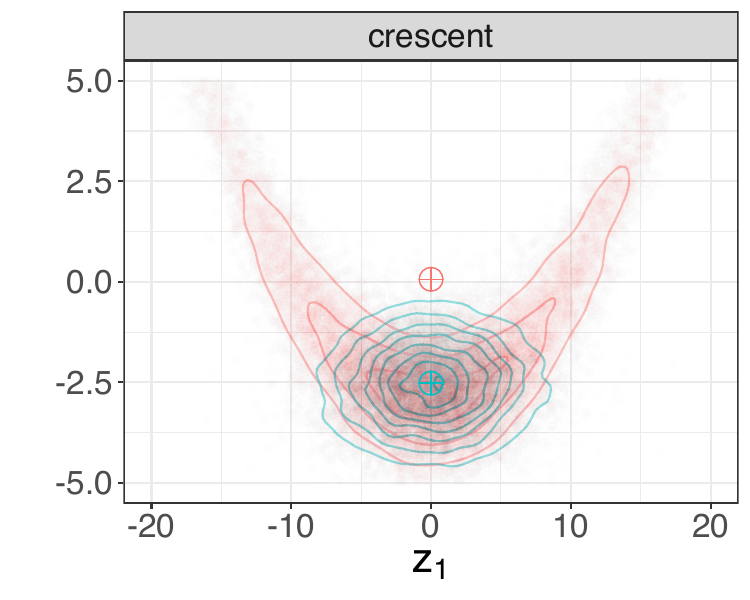}
    \caption{\textit{VI with a Gaussian approximation and dense covariance matrix.
    Left: VI recovers the mean, correlation, and iso-probability contours when the target (student-t) has elliptical symmetry. Right: VI recovers neither the mean nor correlation when the target (crescent) has neither even nor elliptical symmetry.
    }}

    \label{fig:scale_recovery}
\end{figure}

\section{NUMERICAL EXPERIMENTS}  \label{sec:experiments}

We now investigate VI on the collection of target distributions in Table~\ref{tab:targets}. Three of these distributions are synthetic: \texttt{student} is a student-t distribution with elliptical symmetry and heavy tails, while \texttt{mixture} and \texttt{crescent} are not even-symmetric. 
The others in Table~\ref{tab:targets} arise as posteriors of Bayesian models; these include a centered hierarchical model (\texttt{8schools}) with 
a highly asymmetric posterior \citep{Gelman:2013}, a non-centered parameterization of the same model (\texttt{8schools\_nc}) with better geometry, a binomial regression (\texttt{GLM}), a Gaussian process (\texttt{disease\_map}) \citep{Vanhatalo:2019}, and a sparse kernel interaction model (\texttt{SKIM}) \citep{Agrawal:2019}, applied to a genomic data set \citep{Margossian:2020}.
It is unclear how much symmetry is present in the posteriors of the last three models.
Further details of each model can be found in Appendix~\ref{app:targets}.
We fit these targets in \texttt{Stan} with both ADVI and long runs of MCMC, and we use the latter as a stand-in for ground truth.

\begin{table*}
    \centering
    \begin{tabular}{l r l r r r r}
         \rowcolor{gray!20}  {\bf Call name} & $d$ & {\bf Description} & $\varepsilon_{90}$ & $\Delta_\text{mean}$& $\Delta_\text{corr}$ & $\Delta_\text{cov}$ \\
    \texttt{student} & 2 & Elliptical target with heavy tails. & 0 & 0.014 & 2.42e-3 & 0.093  \\
    \rowcolor{gray!20} \texttt{disease\_map} & 102 & Gaussian process model for epidemiology data. & 9.22e-4 & 0.026 & 0.017 & 5.29 \\
    \texttt{GLM} & 3 & Binomial regression for ecological data. & 2.73e-3 & 0.025 & 0.10 & 6.71 \\
    \rowcolor{gray!20} \texttt{8schools\_nc} & 10 & Non-centered hierarchical model for education data. & 0.50 & 0.004 & 0.010 & 1.01 \\
    \texttt{mixture} & 2 & Mixture of two Gaussians with different scales. & 0.54 & 0.26 & 0.27 & 23.40 \\
    \rowcolor{gray!20} \texttt{SKIM} & 305 & Sparse kernel interaction model for genetic data. & 0.82 & 0.12 & 0.014 & 8.30 \\
    \texttt{8schools} & 10 & Centered hierarchical model for education data. & 29.31 & 0.68 & 0.11 & 0.90 \\
    \rowcolor{gray!20} \texttt{Crescent} & 2 & Rosenbrock distribution. & 93.48 & 0.25 & 0.26 & 1.395
    \end{tabular}
    \caption{\textit{Scaled errors ($\Delta_\text{mean},\Delta_\text{corr},\Delta_\text{cov}$) 
    from VI when estimating the mean, correlation, and covariance of the target density, $p$.
    Errors are averaged across all elements of mean vectors and correlation/covariance matrices, and 
    data sets are ordered by the degree to which $p$ violates the assumption of even symmetry, as measured by~$\varepsilon_{90}$ in eq.~(\ref{eq:varepsilon}).
    As a trend, less symmetric targets yield worse estimates of the mean and correlation; also, the correlation can be well-estimated even when the covariance is not. 
    }}
    \label{tab:targets}
\end{table*}


For each distribution in Table ~\ref{tab:targets}, we use a stochastic procedure to assess whether it is approximately even-symmetric.
Given a draw $z$ from $p(z)$---obtained 
by MCMC for the more complex distributions---we compute the \textit{reflected} sample $z' = 2\mu\!-\!z$,
%
%
where $\mu$ is the (estimated) mean of $z$ with respect to the density~$p$. 
For each $(z,z')$ pair, we then compute the violation 
\begin{equation} \label{eq:varepsilon}
    \varepsilon(z) = \left |\frac{\log \pi(z|x) - \log \pi(z'|x)}{\log \pi(z| x)} \right|.
\end{equation}
If $\pi(z|x)$ is even-symmetric about $\mu$, then \mbox{$\varepsilon(z)$} vanishes for all $z$.
We evaluate $\varepsilon(z)$ at many (MCMC) samples and report the 90$^\text{th}$ quantile $\varepsilon_{90}$ in Table~\ref{tab:targets}.

For each target, we also report how well ADVI estimates its mean, correlation, and covariance.
We~use
\begin{equation} \label{eq:scaled}
    \Delta_\text{mean} = \frac{|\mathbb E_p (z) - \mathbb E_q (z) |}{\text{max}(\sqrt{\text{Var}_p (z)}, |\mathbb E_p(z)|)}
\end{equation}
to measure the error of the estimated means. Eq.~(\ref{eq:scaled}) is both scale-free and numerically stable; see Appendix~\ref{app:algorithms} for further justification of this choice.
We also measure the absolute relative error of the covariances and the absolute error of the correlations (which are already scale-free); in Table~\ref{tab:targets} we denote these, respectively, 
by $\Delta_\text{cov}$ and $\Delta_\text{corr}$.

\begin{figure}
    \begin{center}
    \includegraphics[width=3.5in]{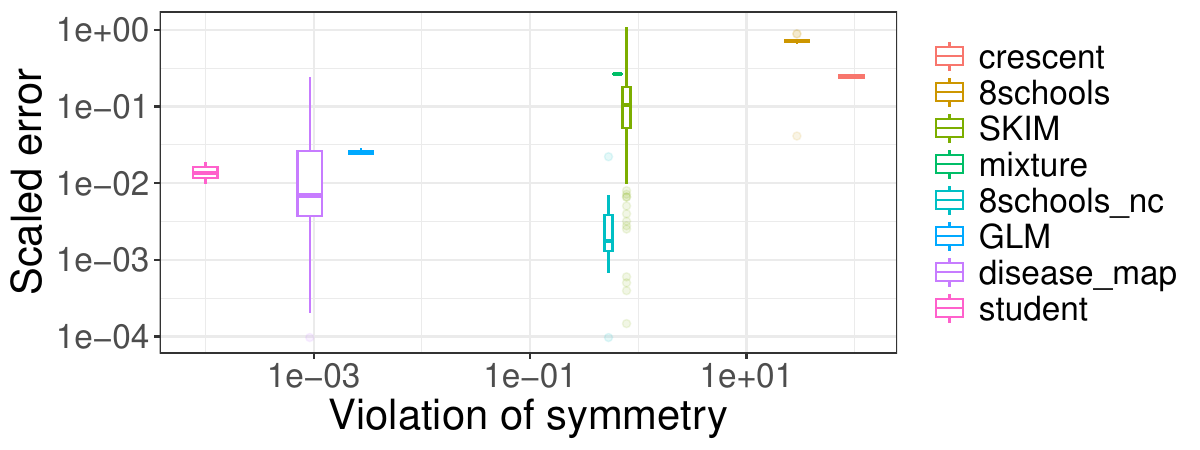}
    \end{center}
   \caption{\textit{Scaled error in VI's estimates of the mean, using a Gaussian approximation, versus the asymmetry of its targets. The box plot summarizes the errors for all latent variables.
    As a trend, VI's estimates are better when its targets are more symmetric (even if they are not Gaussian). See text for details.
    }}
   \label{fig:model_mean}
\end{figure}

The plot reveals a trend: VI yields worse estimates 
as $p$ becomes less symmetric, with the worst estimates obtained for 
the Rosenbrock distribution (\texttt{crescent}) and the centered hierarchical model (\texttt{8schools}). 
Table~\ref{tab:targets} also reports the error in estimates of the correlations and covariances.
Estimates of the former are consistently better than the latter. However, the correlation matrix is estimated poorly for the three models that most clearly violate elliptical symmetry (\texttt{mixture}, \texttt{8schools}, \texttt{crescent}).

In summary, blatantly asymmetric distributions  produce VI estimators with a large error.
Conversely, targets with a near-perfect symmetry yield accurate estimators.
Still, it is unclear from our experiments how varying degrees of asymmetry affect the quality of VI estimators.
Two intriguing cases are the \texttt{8schools\_nc} and \texttt{mixture} models, which have a comparable violation of symmetry, however the error in \texttt{mixture} is noticibly larger.




\section{DISCUSSION}

In this paper, we have derived broad conditions under which VI with location-scale families exactly recovers the mean and the correlation matrix. Our proofs use basic principles of symmetry to illustrate settings in which $\KL(q||p)$ has a unique global minimizer. 
The conditions of our theorems allow for several misspecifications: notably, the approximation $q$ need not match the tail behavior of the target $p$ and, when estimating the mean, $q$ can be factorized even though $p$ is not.

In a Bayesian setting, the mean is most commonly used to summarize the posterior, and the correlation, though less often reported, is also of interest. Conceivably, one could employ the scale matrix $S$ returned by VI as a pre-conditioner for other inference algorithms, such as MCMC, which are known to converge more quickly for (approximately) spherically symmetric distributions. This is a direction for future work.

Our results assume that VI algorithms find a minimizer of $\KL(q||p)$, which in practice may not be true.
It can be challenging to optimize the ELBO  \citep{Dhaka:2020, Dhaka:2021, Agrawal:2021}; however, recent papers provide convergence guarantees for the optimizations in VI \citep{Lambert:2022, Diao:2023, Domke:2023, Kim:2023, Jiang:2023}.
These articles, combined with our paper, ensure that VI converges to useful estimates of its targets.
Some of these articles also provide directions for future research, as they consider families of approximations beyond the location-scale families in this paper.
For example, \citet{Lambert:2022} examine Gaussian mixtures, and \citet{Jiang:2023} study the family of all factorized distributions.
A natural avenue for future work is to analyze the role of symmetries in the variational approximations for these families.

It would also be interesting to study cases where the target distribution $p$ is nearly symmetric or (at the opposite extreme) highly asymmetric.
%
Another direction for future work is to bound the error of VI's estimates in terms of measurable violations of symmetry.
Such an approach would align with statistical analyses of VI in asymptotic regimes \citep[e.g.][]{Zhang:2020, Katsevich:2024}; these analyses have yielded pre-asymptotic error bounds and shown that exact estimates are recovered asymptotically.
A related goal is to construct post-hoc diagnostics which leverage empirical measures of asymmetry (as in Section~\ref{sec:experiments}). Our experiments with asymmetric distributions revealed many cases where the error of VI is unacceptably large, and these cases provides further motivation for skewed variational approximations \citep{Tan:2024}. 

Finally, an intriguing conjecture, also deserving of study, is that for each symmetry of $p$, there is a corresponding statistic that VI under mild conditions can exactly recover.

\subsubsection*{Acknowledgments}
We thank four anonymous referees for their constructive feedback and Isaac Rankin for additional feedback on our manuscript.

\bibliography{ref.bib}

%
%
%
%



\section*{Checklist}
%
%
%
%
%
\begin{enumerate}

\item For all models and algorithms presented, check if you include:
\begin{enumerate}
  \item A clear description of the mathematical setting, assumptions, algorithm, and/or model. [\textbf{Yes}]
  \item An analysis of the properties and complexity (time, space, sample size) of any algorithm. [\textbf{Not Applicable}]
  \item (Optional) Anonymized source code, with specification of all dependencies, including external libraries. [\textbf{Yes}.]
\end{enumerate}

\item For any theoretical claim, check if you include:
\begin{enumerate}
  \item Statements of the full set of assumptions of all theoretical results. [\textbf{Yes}]
  \item Complete proofs of all theoretical results. [\textbf{Yes}]
  \item Clear explanations of any assumptions. [\textbf{Yes}]     
\end{enumerate}

\item For all figures and tables that present empirical results, check if you include:
\begin{enumerate}
  \item The code, data, and instructions needed to reproduce the main experimental results (either in the supplemental material or as a URL). [\textbf{Yes}]
  \item All the training details (e.g., data splits, hyperparameters, how they were chosen). [\textbf{Yes}]
        \item A clear definition of the specific measure or statistics and error bars (e.g., with respect to the random seed after running experiments multiple times). [\textbf{Yes}]
        \item A description of the computing infrastructure used. (e.g., type of GPUs, internal cluster, or cloud provider). [\textbf{Yes}]
\end{enumerate}

\item If you are using existing assets (e.g., code, data, models) or curating/releasing new assets, check if you include:
\begin{enumerate}
  \item Citations of the creator If your work uses existing assets. [\textbf{Yes}]
  \item The license information of the assets, if applicable. [\textbf{Not Applicable}]
  \item New assets either in the supplemental material or as a URL, if applicable. [\textbf{Not Applicable}]
  \item Information about consent from data providers/curators. [\textbf{Not Applicable}]
  \item Discussion of sensible content if applicable, e.g., personally identifiable information or offensive content. [\textbf{Not Applicable}]
\end{enumerate}

\item If you used crowdsourcing or conducted research with human subjects, check if you include:
\begin{enumerate}
  \item The full text of instructions given to participants and screenshots. [\textbf{Not Applicable}]
  \item Descriptions of potential participant risks, with links to Institutional Review Board (IRB) approvals if applicable. [\textbf{Not Applicable}]
  \item The estimated hourly wage paid to participants and the total amount spent on participant compensation. [\textbf{Not Applicable}]
\end{enumerate}

\end{enumerate}

\appendix

\section{SYMMETRY OF LIKELIHOOD} \label{app:symmetry}

\begin{figure*}
    \centering
    \includegraphics[width=0.65\linewidth]{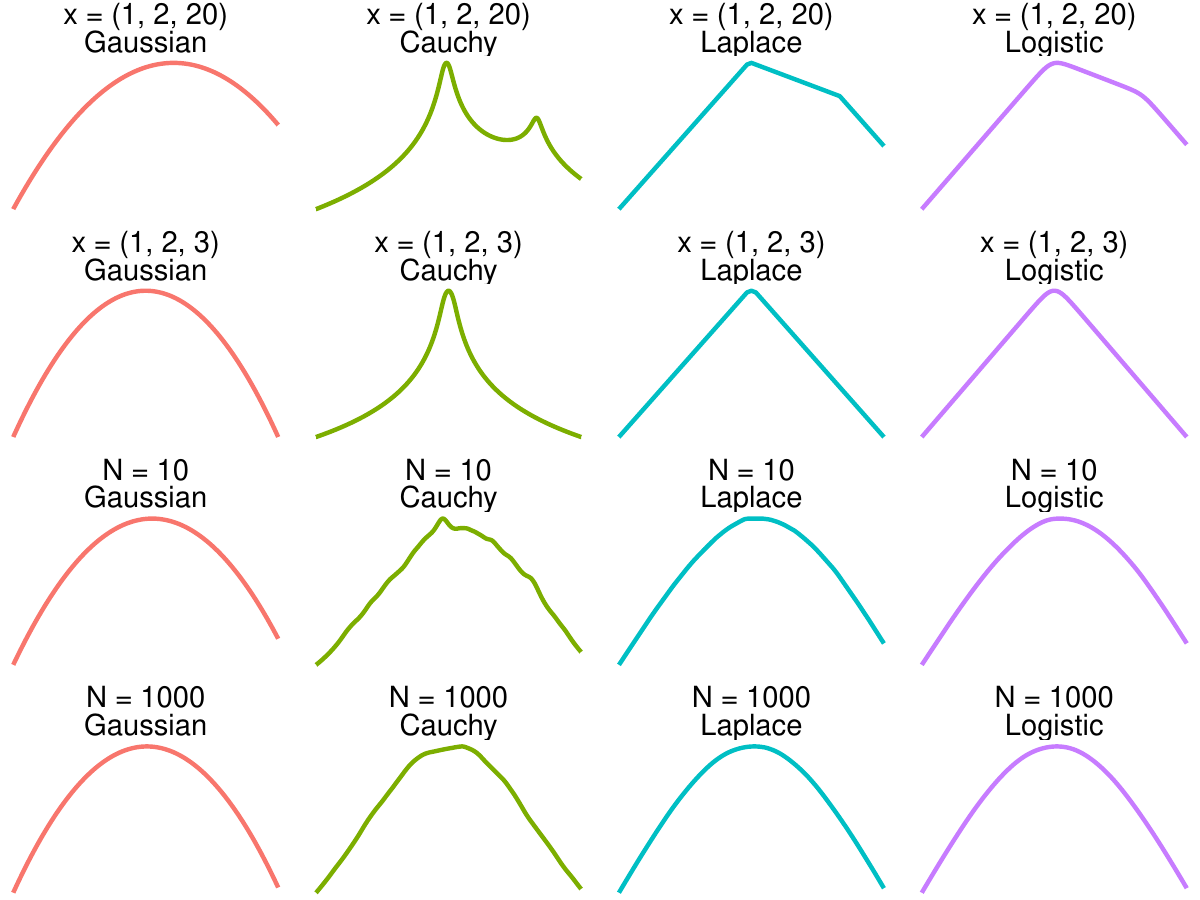}
    \caption{\textit{Symmetry of the log likelihood $f_x(z) = \log \pi(x \mid z)$. In the Gaussian case, the likelihood function is always symmetric about the mean of $x$. Other likelihoods do not in general admit a point of symmetry, as illustrated in the first row. However, if the data is symmetric, the likelihood is also symmetric (second row).
    Drawing data from a symmetric distribution generates samples which are approximately symmetric, and correspondingly approximately symmetric likelihoods.
    }}
    \label{fig:likelihood_symmetry}
\end{figure*}

In this appendix, we empirically investigate the symmetry properties of the likelihood.
We examine independent and identically distributed observations, meaning the likelihood factorizes as
\begin{equation}
    \pi(x \mid z) = \prod_{i} \pi(x_i \mid z).
\end{equation}
Recall that the likelihood is a function of $z$, with $x$ kept fixed. 
We consider the Gaussian, Cauchy, Laplace, and logistic distributions, and take $z$ to be their location parameter.
In all these cases, the single-component likelihood $\pi(x_i \mid z)$ admits a point of symmetry at $x_i$.
However, the product of these single-component likelihoods does not in general have a point of symmetry. 
The exception is the Gaussian case, where the likelihood is always symmetric about the mean of $x$.

A natural condition for the likelihood to be symmetric is for the data itself to be symmetric.
To illustrate this, we generate $N$ observations.
First we create $N = 3$ adversarial samples $x = (-1, 0, 20)$.
These observations are deliberately asymmetric, and produce asymmetric likelihoods (Figure~\ref{fig:likelihood_symmetry}).
Next we generate $N = 3$ ideal samples, $x = (-1, 0, 1)$, which are symmetric and produce symmetric likelihoods.
Finally, we randomly draw $N = 10$ and $N = 100$ observations from a symmetric distribution $p$ (somewhat arbitrarily, we chose $p$ to be Gaussian though any well-behaved symmetric distribution would do).
While no particular realization of the data is exactly symmetric, the data is ``on average'' symmetric.
Already for $N = 10$, we find that all non-Gaussian likelihoods are approximately symmetric, with the symmetry becoming stronger for $N = 1000$.

\section{SUPPORTING PROOFS}
\label{app:lemma-KL}

In this appendix we prove that the KL divergence in eqs.~(\ref{eq:KL-nu}--\ref{eq:KL-convex-nu}) is \textit{strictly} convex in the location parameter~$\nu$. This result for $\KL(q_\nu||p)$ is needed in Theorem~\ref{thm:location} to prove that VI recovers the exact mean, and it is also needed in Theorem~\ref{thm:scale} to prove that VI recovers the exact correlation matrix. 

We begin by proving two simpler propositions: these propositions show that certain basic properties of strictly convex functions extend to functions that are everywhere convex but only strictly convex on some open set. We begin, in the simplest case, by establishing a property of such functions on the real line.

\begin{proposition}
\label{prop:scR1}
Let $f:\mathbb{R}\rightarrow\mathbb{R}$ be differentiable on~$\mathbb{R}$. Also, let $x_0,x_1\!\in\!\mathbb{R}$ with $x_0\!<\!x_1$, and let 
\begin{equation}
x_\lambda = (1-\lambda)x_0 + \lambda x_1
\label{eq:x-eta}
\end{equation}
for some $\lambda\in(0,1)$. 
If $f$ is convex on $\mathbb{R}$ and strictly convex in a neighborhood of~$x_\lambda$, then
\begin{equation}
f(x_\lambda) < (1\!-\!\lambda)f(x_0) + \lambda f(x_1)
\end{equation}
where the above inequality is strict.
\end{proposition}

\begin{proof}
Since $f$ is differentiable, and since $x_0\!<\!x_\lambda\!<\!x_1$, we can write
\begin{align}
f(x_0) &= f(x_\lambda) - \int_{x_0}^{x_\lambda}\!\!f'(\xi)\,\text{d}\xi, 
\label{eq:f0} \\
f(x_1) &= f(x_\lambda) + \int_{x_\lambda}^{x_1}\!\! f'(\xi)\,\text{d}\xi.
\label{eq:f1}
\end{align}
Since $f$ is convex on $\mathbb{R}$, its derivative is everywhere non-decreasing; also, since $f$ is strictly convex in some interval containing~$x_\lambda$, its derivative is strictly increasing on this interval. From these observations and the form of eqs.~(\ref{eq:f0}--\ref{eq:f1}), we deduce that
\begin{align}
f(x_0) &> f(x_\lambda) - (x_\lambda\!-\!x_0) f'(x_\lambda), \\
f(x_1) &> f(x_\lambda) + (x_1\!-\!x_\lambda)  f'(x_\lambda).
\end{align}
Next we use eq.~(\ref{eq:x-eta}) to eliminate $x_\lambda$ from the differences in the above expressions. In this way we find
\begin{align}
f(x_0) &> f(x_\lambda) - \lambda (x_1\!-\!x_0) f'(x_\lambda), \\
f(x_1) &> f(x_\lambda) + (1\!-\!\lambda)(x_1\!-\!x_0)  f'(x_\lambda).
\end{align}
We obtain the desired result by taking the particular convex combination of these inequalities that cancels out their rightmost terms:
\begin{equation}
(1\!-\!\lambda)f(x_0) + \lambda f(x_1) > f(x_\lambda).
\end{equation}
This proves the result.
\end{proof}

The next proposition builds on the previous one to establish strict convexity for certain functions over $\mathbb{R}^d$.

\begin{proposition}
\label{prop:scRd}
Let $q$ be a density with positive support on all of $\mathbb{R}^d$. Let $g,h$ be functions on $\mathbb{R}^d$ related~by
\begin{equation}
h(\nu) = \int\!\! g(\nu\!+\!\zeta)\, q(\zeta)\, \text{d}\zeta.
\label{eq:gq}
\end{equation}
If $g$ is differentiable, convex on $\mathbb{R}^d$, and strictly convex on some open set of $\mathbb{R}^d$, then $h$ is strictly convex on~$\mathbb{R}^d$.
\end{proposition}

\begin{proof}
Let $\nu_0,\nu_1\in\mathbb{R}^d$, and let $\lambda\in(0,1)$. Since $g$ is convex, it follows directly from eq.~(\ref{eq:gq}) that $h$ is also convex, satisfying
\begin{equation}
h((1\!-\!\lambda)\nu_0 + \lambda\nu_1) \leq 
  (1\!-\!\lambda)h(\nu_0) + \lambda h(\nu_1).
\label{eq:h-convex}
\end{equation}
To prove the theorem, we must show additionally that this inequality is \textit{strict}. 

\begin{figure}[t]
\centerline{\includegraphics[width=0.3\textwidth]{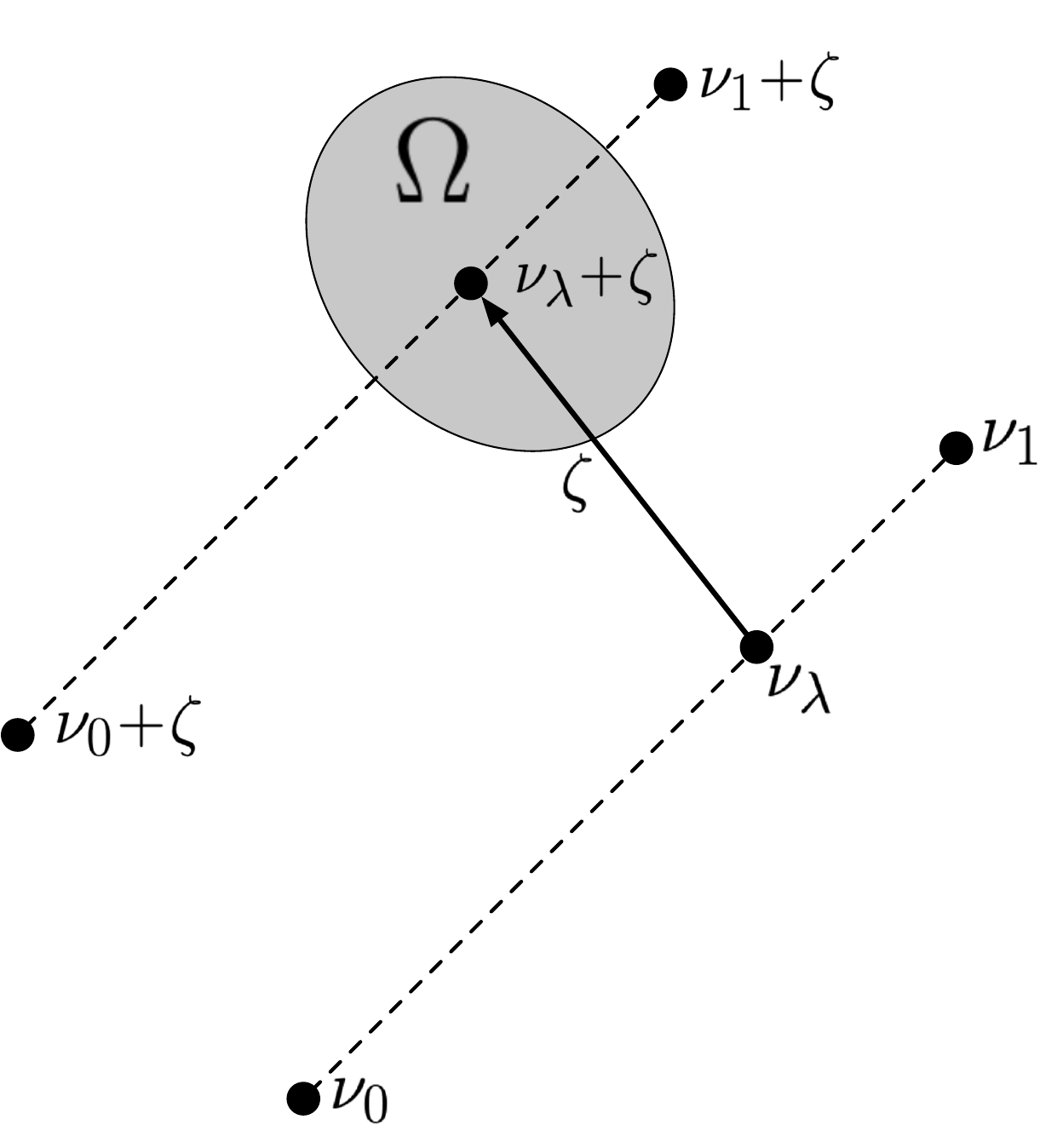}}
\caption{\textit{Illustration of the points $\nu_0$, $\nu_1$, $\nu_\lambda$, the translation $\zeta$, and the open set $\Omega$ in the proof of Proposition~\ref{prop:scRd}.}}
\label{fig:strictly-convex}
\end{figure}

Let $\Omega$ denote the open set of $\mathbb{R}^d$ in which $h$ is strictly convex, and as shorthand, denote the convex combination of $\nu_0$ and $\nu_1$ in eq.~(\ref{eq:h-convex}) by
\begin{equation}
\nu_\lambda = (1\!-\!\lambda)\nu_0 + \lambda\nu_1.
\label{eq:nu-lambda}
\end{equation}
We can define another open set $\Omega_\lambda$ in $\mathbb{R}^d$ by translating the points of $\Omega$ by an amount $-\nu_\lambda$; namely,
\begin{equation}
\Omega_\lambda = \left\{\zeta \in \mathbb R^d\, \big|\, \nu_\lambda\! +\! \zeta \in\Omega\right\}.
\end{equation}
Then from eq.~(\ref{eq:gq}), we can express $h(\nu)$ as the sum of two integrals, one over the set $\Omega_\lambda$, and one over its complement:
\begin{equation}
h(\nu) = \int_{\Omega_\lambda} \!\! g(\nu\!+\!\zeta)\, q(\zeta)\, \text{d}\zeta + \int_{\overline{\Omega}_\lambda} \!\! g(\nu\!+\!\zeta)\, q(\zeta)\, \text{d}\zeta
\label{eq:gq2}
\end{equation}
In what follows we will be especially focused on the first of these integrals.

Let $\zeta\!\in\!\Omega_\lambda$, and consider the illustration of $\nu_0,\nu_1,\nu_\lambda,\zeta$ and $\Omega$ in Fig.~\ref{fig:strictly-convex}. Define the function $f:\mathbb{R}\rightarrow\mathbb{R}$ by
\begin{equation}
f(x) = g((1\!-\!x)\nu_0 + \nu_1 + \zeta).
\label{eq:fg}
\end{equation}
Note that $f$ is convex (since $g$ is convex), and also that by construction, $f$ is strictly convex in some neighborhood of $x\!=\!\lambda$, since in this neighborhood the argument of $g((1\!-\!x)\nu_0 + \nu_1 + \zeta)$ in eq.~(\ref{eq:fg}) belongs to the set~$\Omega_\lambda$. It therefore follows from Proposition~\ref{prop:scR1} that
\begin{equation}
f(\lambda) < (1\!-\!\lambda)f(0) + \lambda f(1),
\end{equation}
or equivalently that
\begin{equation}
g(\nu_\lambda\!+\!\zeta) < (1\!-\!\lambda)g(\nu_0) + \lambda g(\nu_1),
\end{equation}
where the above inequality is strict for all $\zeta\!\in\!\Omega_\lambda$. 

Now we use this inequality to compute a strict lower bound on $h(\nu_\lambda)$. For the first integral in 
eq.~(\ref{eq:gq2}), over the set $\Omega_\lambda$, we have the strict lower bound
\begin{equation}
\int_{\Omega_\lambda} \!\! g(\nu_\lambda\!+\!\zeta)\, q(\zeta)\, \text{d}\zeta 
  < \int_{{\Omega}_\lambda} \!\!\big[(1\!-\!\lambda)g(\nu_0) + \lambda g(\nu_1)\big]\, q(\zeta)\, \text{d}\zeta,
\end{equation}
and for the second integral, over the complement of this set, we have the lower bound
\begin{equation}
\int_{\overline{\Omega}_\lambda} \!\! g(\nu_\lambda\!+\!\zeta)\, q(\zeta)\, \text{d}\zeta 
  \leq \int_{\overline{\Omega}_\lambda} \!\!\big[(1\!-\!\lambda)g(\nu_0) + \lambda g(\nu_1)\big]\, q(\zeta)\, \text{d}\zeta.
\end{equation}
Finally, summing the previous two equations, we obtain the result that
\begin{equation}
h(\nu_\lambda) < 
  (1\!-\!\lambda)h(\nu_0) + \lambda h(\nu_1),
\label{eq:h-strict-convex}
\end{equation}
which is the desired strengthening of eq.~(\ref{eq:h-convex}). The above holds 
for any $\lambda\!\in\!(0,1)$, $\nu_0,\nu_1\!\in\!\mathbb{R}^d$, and~$\nu_\lambda$ given by eq.~(\ref{eq:nu-lambda}); thus $h$ is strictly convex.
\end{proof}

We can now prove the main lemma needed for the proofs of Theorems~\ref{thm:location} and~\ref{thm:scale}.

\begin{lemma}
\label{lemma:KL}
Let $\mathcal{Q}$ be a location family, and for \mbox{$q_\nu\!\in\!\mathcal{Q}$} and some density $p$, consider the divergence $\KL(q_\nu||p)$ as a function of the location parameter $\nu\!\in\!\mathbb{R}^d$. If $\log p$ is concave on all of $\mathbb{R}^d$ and strictly concave on some open set of $\mathbb{R}^d$, then $\KL(q_\nu||p)$ is a strictly convex function of the location parameter, $\nu$.
\end{lemma}

\begin{proof}
From eq.~(\ref{eq:KL-convex-nu}), the KL divergence is given by
\begin{equation}
\KL(q_\nu||p) 
 = -\mathcal{H}(q_0)\, - \int\! \log p(\nu\!+\!\zeta)\, q_0(\zeta)\,\text{d}\zeta.
\end{equation}
Note that the entropy in this expression, $\mathcal{H}(q_0)$, has no dependence on the location parameter. Thus the lemma follows directly from the previous proposition and the concavity assumptions on $\log p$.
\end{proof}

\section{EXPERIMENTAL DETAILS}

In this appendix, we provide additional details for all  experiments in the paper.

\subsection{Algorithms and evaluation metric} \label{app:algorithms}

For 1-dimensional examples, we minimize $\KL(q||p)$ using a grid-search.

For all other examples, we use the statistical software \texttt{Stan} \citep{Carpenter:2017}, which supports gradient-based algorithms, and the add-on \texttt{bridgeStan} \citep{Roualdes:2023}, which enables convenient manipulations of \texttt{Stan}'s output.
\texttt{Stan} transforms all constrained latent variables to the unconstrained scale and so all algorithms, including VI and MCMC, operate over $\mathbb R^d$.
We report the quality of the approximation over the unconstrained space.
By working in this unconstrained space, we ensure that 
the experiments align with our theoretical analysis, which assumes that $p$ and $q$ are defined over $\mathbb R^d$.

All experiments are conducted with a 2.8 GHz Quad-Core Intel Core i7 CPU processor.
We use the command line interface of \texttt{Stan}, specifically \texttt{cmdStan} v2.34.1.
As a scripting language (i.e. for data manipulation and figures), we use \texttt{R} v4.2.3.

\subsubsection{Variational inference}

In our experiments we use automatic differentiation~VI \citep[ADVI;][]{Kucukelbir:2017} where $\mathcal Q$ is chosen to be a family of Gaussian distributions.
In Section~\ref{sec:location-illustration} we use the ``mean-field'' mode, meaning the covariance matrix of $q$ is diagonal;
for all other experiments, we use the ``fullrank'' mode and estimate the full covariance matrix.
ADVI optimizes the ELBO using stochastic gradient descent, as described in Section~\ref{sec:VI}.
In order to stabilize the algorithm, we set the default number of Monte Carlo draws to 1,000 per optimization step.
For certain problems, we depart from this default, either to stabilize the solution or, for larger problems, to reduce compute time.
In particular, for \texttt{disease\_map} and \texttt{SKIM} we reduce the number of Monte Carlo samples per optimization step to 50.

\subsubsection{Markov chain Monte Carlo benchmark}

As a benchmark, we run long chains of MCMC.
Specifically, we use \texttt{Stan}'s default sampler, which is a dynamic Hamiltonian Monte Carlo algorithm \citep{Betancourt:2018, Hoffman:2014}.
We use 1,000 warmup iterations and 20,000 sampling iterations.

\subsubsection{Evaluation metrics}

To assess the quality of VI's approximation, we examine 
the posterior mean, the posterior correlation matrix, and the posterior covariance matrix.
For the mean, we compute
\begin{equation}  \label{eq:delta_mean}
    \Delta_\text{mean} = \frac{|\mathbb E_p(z) - \mathbb E_q(z)|}{\text{max}(\sqrt{\text{Var}_p(z)}, |\mathbb E_p(z)|)},
\end{equation}
where the scaling factor is chosen to make the error scale-free and numerically stable.
Note that we do not simply compute the relative absolute error, with $\mathbb{E}_p(z)$ in the denominator: though a natural choice, 
it is unstable when $\mathbb E_p(z) \approx 0$.
We also do not simply divide by the posterior standard deviation, as this quantity decays to 0 in cases where the data is rich (e.g., \texttt{glm}), and this in turn causes any errors in the numerator to be severely penalized. 
As a result, the performance of VI---when measured relative to the posterior standard deviation---decreases with the number of observations and this effect must be counterbalanced by fine-tuning the optimizer.
The metric in eq.~\eqref{eq:delta_mean} was used to safeguard our analysis against these sources of numerical instability. 

We report the errors in the correlation and covariance in slightly different ways.
For the correlation, we simply report the absolute error
\begin{equation}
    \Delta_\text{corr} = |\text{Corr}_p(z_i, z_j) - \text{Corr}_q(z_i, z_j)|,
\end{equation}
noting that the correlation is scale-free.
On the other hand, for the covariance, we report the absolute relative error
\begin{equation}
    \Delta_\text{cov} = \frac{|\text{Cov}_{p(z)} - \text{Cov}_{q(z)}|}{|\text{Cov}_{p(z)}|}.
\end{equation}
This metric suffers from the usual problem that for covariances close to 0, the error is numerically unstable.
We considered adjusting the denominator, as was done for $\Delta_\text{mean}$, however we do not have good estimates of the variance's variance (i.e. the fourth moment), even with long runs of MCMC.
We also did not find that the reported covariance errors were unstable.

\subsection{Additional results for Bayesian logistic regression} \label{app:logistic}

\begin{figure*}
    \centering
    \includegraphics[width=1\linewidth]{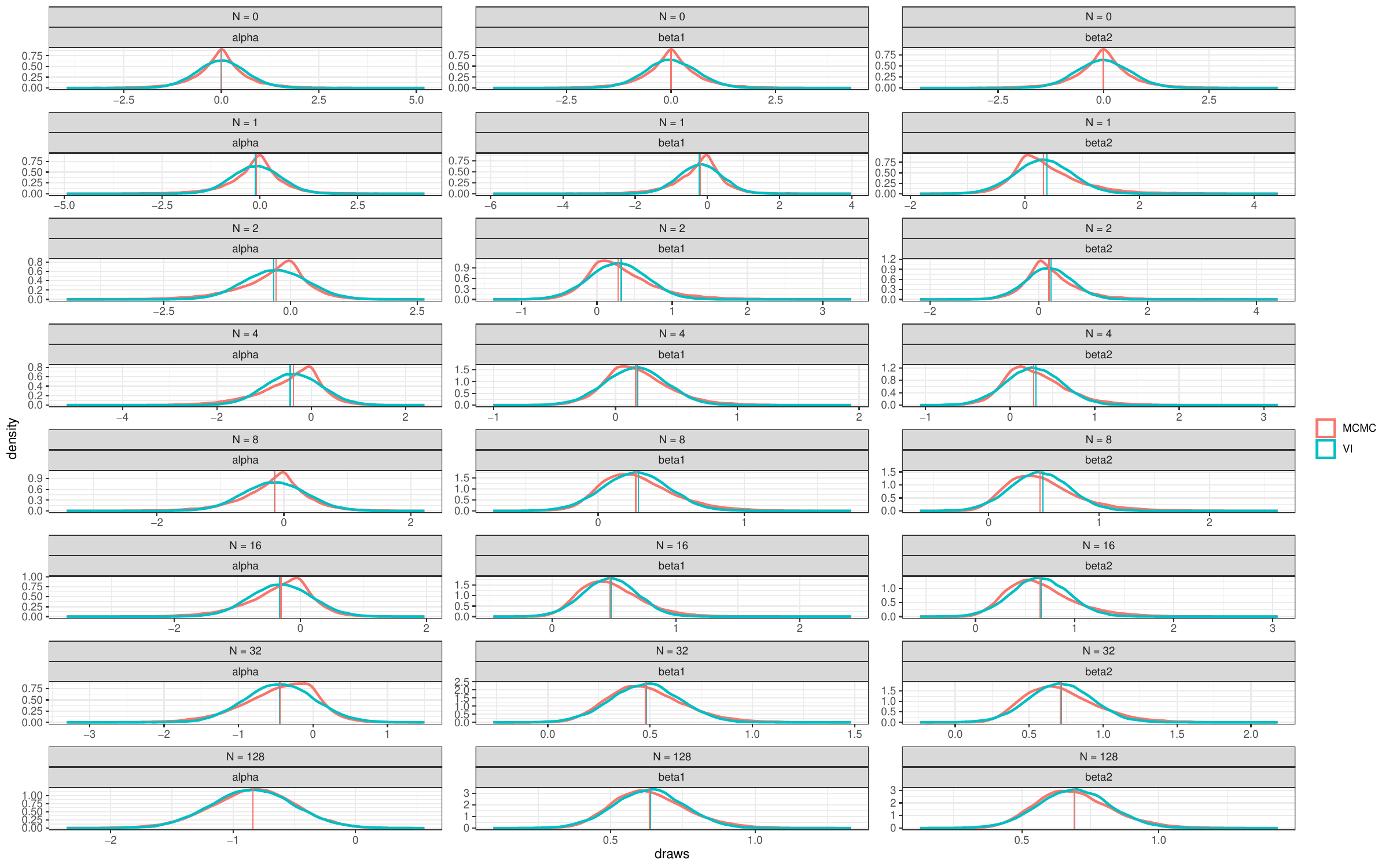}
    \caption{\textit{Posterior distributions for $\alpha$, $\beta_1$, and $\beta_2$ for a Bayesian logistic regression with $N$ observations.
    Vertical lines indicate the means estimated by MCMC and VI.
    Overall, we find that these estimates match when the posterior is symmetric ($N = 0, N = 128$).
    In the presence of assymmetry, VI can still correctly estimate the marginal mean for some components, though not all of them.
    }}
    \label{fig:logistic_density_all}
\end{figure*}

Figure~\ref{fig:logistic_density_all} plots the marginal posterior densities for the Bayesian logistic regression introduced in section~\ref{sec:posterior-symmetry} with $N\! =\! 0, 1, 2, 4, 8, 16, 32, 128$ observations.
(For clarity, the main body only shows the marginal posteriors for $\beta_1$ with $N\! =\! 0, 4, 128$.
From the figure, we see how the asymmetry manifests for all regression parameters as the number of examples is increased from $N\!=\!0$ (where the prior dominates) to $N\!=\!128$) (where the likelihood dominates).
When the posterior is asymmetric, the posterior mean can still be accurately estimated (e.g., $N\!=\! 32$), and in many cases we see that one component (but not all three) of the posterior mean is recovered.

\subsection{Targets in numerical experiments} \label{app:targets}

In this appendix, we provide details on the targets used in section~\ref{sec:experiments}.

\texttt{student-t} ($d=2$). A multivariate student-t distribution with correlation 0.5 between $z_1$ and $z_2$.

\texttt{disease\_map} ($d=102$). A disease map of Finland to model mortality counts across counties \citep{Vanhatalo:2019}. 
The model uses a Gaussian process model with an exponentiated squared kernel,
  \begin{equation}
     k(x_i, x_j) = \alpha^2 \exp \left (- \frac{(x_i - x_j)^T(x_i - x_j)}{\rho^2} \right),
  \end{equation}
where $x_i$ is the two-dimensional coordinate of the $i^\text{th}$ county.
The covariance matrix $K$ of the Gaussian process is then given by $K_{ij} = k(x_i, x_j)$.
The full model is
\begin{eqnarray}
    \rho & \sim & \mathrm{invGamma}(2.42, 14.81), \nonumber \\
    \alpha & \sim & \mathrm{invGamma}(10, 10), \nonumber \\
    \epsilon & \sim & \mathcal N(0, I_{n \times n}), \nonumber \\
    L & = & \mathrm{Cholesky \ decompose}(K), \nonumber \\
    \theta & = & L \epsilon, \nonumber \\
    y_i & \sim & \mathrm{Poisson}(y^i_e e^{\theta_i}).
\end{eqnarray}

\texttt{GLM} ($d=3$). Binomial general linear model for modeling the success rate of Peregrine broods in the French Jura.
This model is part of the model data base \texttt{PosteriorDB} \citep{Magnusson:2024}
The observations are:
\begin{itemize}
    \item $N$, the number of surveyed Peregrine falcons,
    \item $C$, the number of Peregrine falcons brooding,
    \item $ye$, year covariate,
\end{itemize}
and the full model is
\begin{eqnarray}
    \alpha & \sim & \mathcal N(0, 100^2) \nonumber \\
    \beta_1 & \sim & \mathcal N(0, 100^2) \nonumber \\
    \beta_2 & \sim & \mathcal N(0, 100^2) \nonumber \\
    C & \sim & \mathcal B \left (N, \text{logit}^{-1} (\alpha + \beta_1 ye + \beta_2 ye^2) \right),
\end{eqnarray}
where $\mathcal B(N, p)$ denotes the binomial distribution with total number of trials $N$ and probability of success $p$.

\texttt{8schools} and \texttt{8schools\_nc} ($d=10$). A Bayesian hierarchical model of the effects of a test preparation program across $N = 8$ schools \citep{Rubin:1981, Gelman:2013}.
For each school we observe $y_i$, the average change in test scores, and $\sigma_i$, the sample standard deviation across students.
The natural (centered) parameterization (\texttt{8schools}) is
\begin{eqnarray}
    \mu & \sim & \mathcal N(5, 3^2) \nonumber \\
    \tau & \sim & \mathcal N^+(0, 5^2) \nonumber \\
    \theta_i & \sim & \mathcal N(\mu, \tau^2) \nonumber \\
    y_i & \sim & \mathcal N(\theta_i, \sigma_i^2),
\end{eqnarray}
where $\mathcal N^+$ denotes a Gaussian distribution truncated at 0.
The latent variables are $z = (\mu, \log \tau, \theta_{1:N})$.
This model is known to exhibit an unbounded posterior density with a funnel shape \citep{Neal:2001}.
The non-centered parameterization (\texttt{8schools\_nc}) alleviates this geometric pathology:
\begin{eqnarray}
    \mu & \sim & \mathcal N(5, 3^2) \nonumber \\
    \tau & \sim & \mathcal N^+(0, 5^2) \nonumber \\
    \epsilon_i & \sim & \mathcal N(0, 1) \nonumber \\
    \theta_i & = & \mu + \tau \epsilon_i \nonumber \\
    y_i & \sim & \mathcal N(\theta_i, \sigma_i^2).
\end{eqnarray}
This data generative process is the same as above, however in this parameterization, the latent variables are $z = (\mu, \log \tau, \epsilon_{1:N})$.

\texttt{mixture} ($d=2$). A balanced mixture of two two-dimensional Gaussians with different scales and no correlation between components.
\begin{equation}
    z_1, z_2 \overset{\text{iid}}{\sim} \frac{1}{2} \mathcal N(-1, 2) + \frac{1}{2} \mathcal N(3, 1).
\end{equation}
This distribution is neither elliptically nor even-symmetric.

\texttt{SKIM} ($d=305$).  The sparse kernel interaction model (SKIM), developed by \citet{Agrawal:2019}, is a regularized regression model which accounts for interaction effects between covariates.
The model makes a soft selection of regression coefficients using a regularized horseshoe prior \citep{Piironen:2017}.
Following \citet{Margossian:2020}, we apply the model to a  genetic microarray classification data set on prostate cancer.
The data set is made up of $N = 102$ patients and $p = 200$ pre-selected genetic covariates.
We denote $y \in \{0, 1\}^N$ the binary observations (1: the patient has cancer, 0: the patient has no cancer), and $X \in \mathbb R^{N \times p}$ the design matrix.

To specify the full data generating process, we first set the following hyperparameters:
\begin{eqnarray}
    p_0 & = & 5 \nonumber \\
    s_\mathrm{global} & = & \frac{p_0}{\sqrt{N}(p - p_0)} \nonumber \\
    \nu_\mathrm{local} & = & 1 \nonumber \\
    \nu_\mathrm{global} & = & 1 \nonumber \\
    s_\mathrm{slab} & = & 2 \nonumber \\
    s_\mathrm{df} & = & 100 \nonumber \\
    c_0 & = & 5.
\end{eqnarray}

Then,
\begin{eqnarray}
  \lambda_i & \sim & \mathrm{Student}_t(\nu_\mathrm{local}, 0, 1) \nonumber \\
     \tau & \sim &  \mathrm{Student}_t(\nu_\mathrm{global}, 0, s_\mathrm{global}) \nonumber  \\
     c_\mathrm{aux} & \sim & \mathrm{inv}\Gamma(s_\mathrm{df} / 2, s_\mathrm{df} / 2) \nonumber \\
  \chi &\sim & \text{InverseGamma}(s_\mathrm{df} / 2, s_\mathrm{df} / 2) \nonumber \\
  c & = & \sqrt{c_\text{aux}} s_\text{slab} \nonumber \\
  \tilde \lambda_i^2 & = &  \frac{c^2 \lambda_i^2}{c^2 + \tau^2 \lambda_i^2} \nonumber \\
 \eta_2 & = & \tau^2 \chi / c^2 \nonumber \\
 \beta_0 & \sim & \mathcal N(0, c_0^2) \nonumber \\
\beta_{i} & \sim & \mathcal N(0, \tau^2 \tilde{\lambda}_i^2) \nonumber  \\
\beta_{ij} & \sim & \mathcal N(0, \eta_2^2 \tilde{\lambda}_i^2 \tilde{\lambda}_j^2) \nonumber \\
y & \sim & \text{Bernoulli}(\text{logit}^{-1}(\beta_0 + X \beta)).
\end{eqnarray}
Following \citet{Agrawal:2019}, we marginalize out $\beta_i$ and $\beta_{ij}$, using a kernel trick and a Gaussian process reparameterization.
To define the Gaussian process covariance matrix $K$, we first introduce the auxiliary matrices:
\begin{eqnarray}
    K_1 & = & X \ \mathrm{diag}(\tilde{\lambda}^2) \ X^T \nonumber \\
    K_2 & = & [X \circ X] \ \mathrm{diag}(\tilde{\lambda}^2) \ [X \circ X]^T,
  \end{eqnarray} 
   where ``$\circ$'' denotes the element-wise Hadamard product.
   Finally,
   \begin{eqnarray}
    K & = & \frac{1}{2} \eta_2^2 (K_1 + 1) \circ (K_1 + 1) - \frac{1}{2} \eta_2^2 K_2
    - (\tau^2 - \eta_2^2) K_1 \nonumber \\ 
    & & + c_0^2  - \frac{1}{2} \eta_2^2.
\end{eqnarray}
and
\begin{eqnarray}
    \epsilon & \sim & \mathcal N(0, 1) \nonumber \\
    L & = & \text{Cholesky\_decompose}(K) \nonumber \\
    f & = & L \epsilon \nonumber \\
    y & \sim & \text{Bernoulli}(\text{logit}^{-1}(f)). 
\end{eqnarray}

\texttt{Crescent} ($d=2$). See Eq.~\eqref{eq:crescent}.

\end{document}